\documentclass[runningheads]{llncs}
\usepackage{graphicx}
\usepackage[utf8]{inputenc}
\usepackage{booktabs}
\usepackage{soul}
\usepackage{url}
\usepackage{hyperref}
\hypersetup{hidelinks}

\usepackage{amsmath}
\usepackage{amssymb}
\usepackage{algorithm}
\usepackage{algpseudocode}
\usepackage[switch]{lineno}

\usepackage{tikz}
\usetikzlibrary{positioning, shapes.geometric, arrows.meta}

\usepackage{subcaption}
\usepackage{wrapfig}

\usepackage{newcommand}
\usepackage{Explanation_font}

\newcommand{\shrink}[1]{}

\usepackage{fancyhdr}
\fancypagestyle{firstpage}
{
    \fancyhead[L]{{\fontsize{8}{0} \selectfont In Proceedings of 18th European Conference on Logics in Artificial Intelligence, JELIA 2023}}   
    \fancyhead[R]{}
}

\begin{document}
\title{A New Class of Explanations for Classifiers with Non-Binary Features}
%
%\titlerunning{Abbreviated paper title}
% If the paper title is too long for the running head, you can set
% an abbreviated paper title here
%
\author{Chunxi Ji\orcidID{0000-0002-4475-1987} \and
Adnan Darwiche\orcidID{0000-0003-3976-6735}}
\authorrunning{C. Ji et al.}
% First names are abbreviated in the running head.
% If there are more than two authors, 'et al.' is used.
%
\institute{University of California, Los Angeles, CA 90095, USA}
\maketitle              % typeset the header of the contribution
\begin{abstract}
Two types of explanations have been receiving increased attention in the literature when analyzing the decisions made
by classifiers. The first type explains why a decision was made and 
is known as a sufficient reason for the decision, also an abductive 
explanation or a PI-explanation.
The second type explains why some other decision was not made and is known
as a necessary reason for the decision, also a contrastive or counterfactual explanation.
These explanations were defined for classifiers with binary, discrete
and, in some cases, continuous features. We show that these
explanations can be significantly improved in the presence of non-binary features, leading to a new class of explanations that
relay more information about decisions and the underlying classifiers.
Necessary and sufficient reasons were also shown to be the prime implicates
and implicants of the complete reason for a decision, which can be obtained
using a quantification operator. We show that our improved notions of necessary and sufficient reasons are also prime implicates and implicants but
for an improved notion of complete reason obtained by a new quantification operator that we also define and study.

\end{abstract}

\thispagestyle{firstpage}

\section{Introduction}

Explaining the decisions of classifiers
has been receiving significant attention in the AI literature recently. 
Some explanation methods operate
directly on classifiers, e.g.,~\cite{ANCHOR,LIME}, while some other
methods operate on symbolic encodings of their 
input-output behavior, e.g.,~\cite{DBLP:conf/cikm/BoumazouzaAMT21,IgnatievNM19a,DBLP:conf/aaai/0001I22,NarodytskaKRSW18}, 
which may be compiled into tractable circuits~\cite{uai/ChanD03,ijcai/ShihCD18,aaai/ShihCD19,kr/ShiSDC20,kr/AudemardKM20,corr/abs-2107-01654}. 
When explaining the decisions of classifiers, two particular notions have been receiving
increased attention in the literature: The \text{sufficient} and \text{necessary} reasons for a decision on an instance.

A \textit{sufficient reason} for a decision~\cite{ecai/DarwicheH20} 
is a minimal subset of the instance which is guaranteed to trigger the decision. It was first introduced under the name \textit{PI-explanation} in~\cite{ijcai/ShihCD18} and later called an \textit{abductive explanation}~\cite{IgnatievNM19a}.\footnote{We will use sufficient reasons and
PI/abductive explanations interchangeably.}
Consider the classifier in Figure~\ref{fig:disease-dg} and a patient, Susan, with
the following characteristics:
${\Age}\GE{55}, \eql{\BType}{\tA}$ and $\eql{\Weight}{\OWeight}$.
Susan is judged as susceptible to disease by this classifier, and a sufficient reason for this decision is $\{{\Age}\GE{55}, \eql{\BType}{\tA}\}$. Hence, the classifier will judge Susan as susceptible to disease as long as she has these two characteristics, regardless of how the feature
$\Weight$ is set.\footnote{See, 
e.g.,~\cite{ijar/ChoiXD12,ANCHOR,ijcai/WangKB21} for some approaches that can be viewed as approximating sufficient reasons
and~\cite{JoaoApp} for a study of the quality of some of these approximations.}

A \textit{necessary reason} for a decision~\cite{DBLP:conf/aaai/DarwicheJ22} 
is a minimal subset of the instance that will flip the
decision if changed appropriately.
It was formalized earlier in~\cite{aiia/IgnatievNA020} under the
name \textit{contrastive explanation} which is discussed initially in~\cite{lipton_1990,DBLP:journals/ai/Miller19}.\footnote{We will use necessary reasons and contrastive explanations interchangeably in this paper. Counterfactual explanations are related but have alternate definitions in the literature. For example, as defined in~\cite{kr/AudemardKM20}, they correspond to length-minimal necessary reasons; see~\cite{DBLP:conf/aaai/DarwicheJ22}. But according to some other definitions, they include contrastive explanations (necessary reasons) as a special case; see Section~5.2 in~\cite{DBLP:journals/logcom/LiuL23}. See also~\cite{DBLP:conf/ijcai/AlbiniRBT20} for counterfactual explanations that are directed towards Bayesian network classifiers and \cite{ijar/Amgoud23} for a relevant recent study and survey.}
Consider again the patient Susan and the classifier in Figure~\ref{fig:disease-dg}.
A necessary reason for the decision on Susan is 
$\{{\Age}\GE{55}\}$, which means that she would not be judged as susceptible to disease if she were younger than $55$. The other necessary reason is $\{\eql{\Weight}{\OWeight}, \eql{\BType}{\tA}\}$ so the decision on Susan can be flipped by changing these two characteristics (and this cannot be achieved by changing only one of them). Indeed, if Susan had $\eql{\Weight}{\Nom}$ and $\eql{\BType}{\tAB}$, she will not be judged as susceptible. However, since $\Weight$ and $\BType$ are discrete variables, there are multiple ways for changing them and some changes may not flip the decision (e.g., $\eql{\Weight}{\UWeight}$ and $\eql{\BType}{\tB}$).

\begin{figure}[tb]
        \centering
        \begin{minipage}[t]{0.45\textwidth}
        \centering
        \scalebox{0.6}{
        \begin{tikzpicture}[
        roundnode/.style={text width = 0.55cm, circle ,draw=black, thick, text badly centered},
        squarednode/.style={rectangle, draw=black, thick, text badly centered},
        ]
        %Nodes
        \node[squarednode]      (Age)                              {\large \fAge};
        \node[squarednode]        (Weight)       [below=of Age, xshift = -0.9cm, yshift = 0.4cm] {\large \fWeight};
        \node[roundnode]        (No1)       [below=of Age, xshift = 0.9cm, yshift = 0.5cm] {\fno};
        \node[roundnode]        (Yes1)       [below=of Weight, xshift = -2.3cm, yshift = 0.2cm] {\fyes};
        \node[squarednode]        (BloodType1)       [below=of Weight, yshift = 0cm] {\large \fBType};
        \node[squarednode]        (BloodType2)       [below=of Weight, xshift = 3cm, yshift = 0cm] {\large \fBType};
        \node[roundnode] (Yes2) [below=of BloodType1, xshift = -1.3cm, yshift = 0.3cm] {\fyes};
        \node[roundnode] (No2) [below=of BloodType1, xshift = 0cm, yshift = 0.3cm] {\fno};
        \node[roundnode] (Yes3) [below=of BloodType2, xshift = -1cm, yshift = 0.3cm] {\fyes};
        \node[roundnode] (No3) [below=of BloodType2, xshift = 0.5cm, yshift = 0.3cm] {\fno};
        
        %Lines
        \draw[-latex, thick] (Age.240) -- node [anchor = center, xshift = -5mm, yshift = 1mm] {\large $\geq 55$} (Weight.north);
        \draw[-latex, thick] (Age.300) -- node [anchor = center, xshift = 5mm, yshift = 1mm] {\large $< 55$} (No1.north);
        \draw[-latex, thick] (Weight.240) --  node [anchor = center, xshift = -12mm] {\large \fOWeight} (Yes1.north);
        \draw[-latex, thick] (Weight.270) --  node [anchor = center, xshift = -0mm] {\large \fUWeight} (BloodType1.north);
        \draw[-latex, thick] (Weight.300) --  node [anchor = center, xshift = 10mm] {\large \fNom} (BloodType2.north);
        \draw[-latex, thick] (BloodType1.240) --  node [anchor = center, xshift = -9mm] {\large \ftA, \ftB, \ftAB} (Yes2.north);
        \draw[-latex, thick] (BloodType1.270) --  node [anchor = center, xshift = 3mm] {\large \ftO} (No2.north);
        \draw[-latex, thick] (BloodType2.240) --  node [anchor = center, xshift = -6mm] {\large \ftA, \ftB} (Yes3.north);
        \draw[-latex, thick] (BloodType2.270) --  node [anchor = center, xshift = 6mm] {\large \ftAB, \ftO} (No3.north);
        \end{tikzpicture}
        }
        \subcaption{\label{fig:disease-dg}}
        \end{minipage}%
        \quad
        \begin{minipage}[t]{0.45\textwidth}
        \centering
        \scalebox{0.6}{
        \begin{tikzpicture}[
        roundnode/.style={text width = 0.55cm, circle ,draw=black, thick, text badly centered},
        squarednode/.style={rectangle, draw=black, thick, text badly centered},
        ]
        %Nodes
        \node[squarednode]      (Age)                              {\large \fAge};
        \node[squarednode]        (Weight)       [below=of Age, xshift = -1.85cm, yshift = 0.4cm] {\large \fWeight};
        \node[squarednode]        (BloodType3)       [below=of Age, xshift = 1.85cm, yshift = 0.4cm] {\large \fBType};
        \node[roundnode] (Yes3) [below=of BloodType3, xshift = -0.3cm, yshift = 0.1cm] {\fyes};
        \node[roundnode] (No3) [below=of BloodType3, xshift = 0.8cm, yshift = 0.1cm] {\fno};
        \draw[-latex, thick] (BloodType3.240) --  node [anchor = center, xshift = -2mm] {\large \ftB} (Yes3.north);
        \draw[-latex, thick] (BloodType3.270) --  node [anchor = center, xshift = 8mm] {\large \ftA, \ftAB, \ftO} (No3.north);
        
        \node[roundnode]        (Yes1)       [below=of Weight, xshift = -2.25cm, yshift = 0.1cm] {\fyes};
        \node[squarednode]        (BloodType1)       [below=of Weight] {\large \fBType};
        \node[roundnode]        (BloodType2)       [below=of Weight, xshift = 2.25cm, yshift = 0.1cm] {\fyes};
        \node[roundnode] (Yes2) [below=of BloodType1, xshift = -1.5cm, yshift = 0.3cm] {\fyes};
        \node[roundnode] (No2) [below=of BloodType1, xshift = 0cm, yshift = 0.3cm] {\fno};
        
        %Lines
        \draw[-latex, thick] (Age.240) -- node [anchor = center, xshift = -8mm, yshift = 1mm] {\large $\geq 55$} (Weight.north);
        \draw[-latex, thick] (Age.300) -- node [anchor = center, xshift = 8mm, yshift = 1mm] {\large $< 55$} (BloodType3.north);
        \draw[-latex, thick] (Weight.240) --  node [anchor = center, xshift = -12mm] {\large \fOWeight} (Yes1.north);
        \draw[-latex, thick] (Weight.270) --  node [anchor = center, xshift = -0mm] {\large \fUWeight} (BloodType1.north);
        \draw[-latex, thick] (Weight.300) --  node [anchor = center, xshift = 10mm] {\large \fNom} (BloodType2.north);
        \draw[-latex, thick] (BloodType1.240) --  node [anchor = center, xshift = -6mm] {\large \ftA, \ftO} (Yes2.north);
        \draw[-latex, thick] (BloodType1.270) --  node [anchor = center, xshift = 5mm] {\large \ftB, \ftAB} (No2.north);
        \end{tikzpicture}
        }
        \subcaption{\label{fig:alternative-dg}}
        \end{minipage}
        \caption{Two classifiers of patients susceptible to a certain disease. The classifier in (b) will be discussed later in the paper.}
\end{figure}
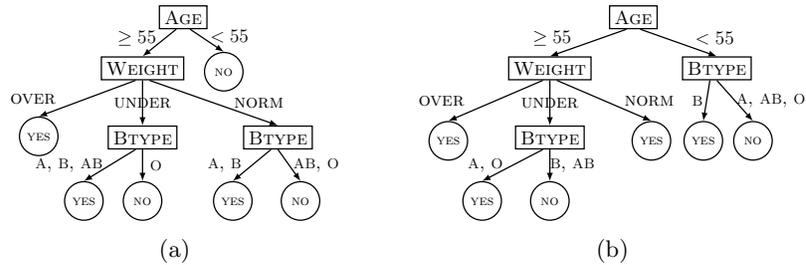

The notion of a \textit{complete reason} behind a decision was introduced in~\cite{ecai/DarwicheH20} and its prime implicants were shown to be the sufficient reasons for the decision. Intuitively, the complete reason is a particular condition on the instance that is both necessary and sufficient for the decision on that instance; see~\cite{lics/Darwiche23}. A declarative semantics for complete reasons was given in~\cite{jair/DarwicheM21} which showed how to compute them 
using \textit{universal literal quantification}.
Furthermore, the prime implicates of a complete reason where shown to be the necessary reasons for the decision in~\cite{DBLP:conf/aaai/DarwicheJ22}.
Given these results, one would first use universal literal quantification to obtain the complete reason for a decision and then compute its prime implicates and implicants to obtain necessary and sufficient explanations. 

Necessary and sufficient reasons 
are \textit{subsets} of the instance being
explained so each reason corresponds to a set of
variable settings (Feature=Value), like $\eql{\Weight}{\UWeight}$ and $\eql{\BType}{\tB}$, which we shall call
\textit{simple literals.}
Since necessary and sufficient reasons correspond to
sets of simple literals, we will refer to them as \textit{simple}
or \textit{classical} explanations. We will show next
that these simple explanations
can be significantly improved 
if the classifier has non-binary features, leading
to more general notions of necessary, 
sufficient and complete reasons that 
provide more informative explanations of decisions.

Consider again the decision on Susan discussed above which
had the sufficient reason $\{\Age\GE 55, \eql{\BType}{\tA}\}$. 
Such an explanation can be viewed as a \textit{property} of the instance
which guarantees the decision. The property has a specific form: a conjunction
of feature settings (i.e., instance characteristics) which leaves out characteristics
of the instance that are irrelevant to the decision ($\eql{\Weight}{\OWeight}$). However, the following is a weaker property of the instance which will also trigger the decision:
$\{\Age\GE 55,\BType\IN\{\tA,\tB\}\}$. This property
tells us that not only is $\eql{\Weight}{\OWeight}$ irrelevant to the decision, but
also that $\eql{\BType}{\tA}$ is not particularly  relevant since $\BType$ could have been $\tB$
and the decision would have still been triggered.
In other words, what is really relevant is that $\BType \IN \{\tA, \tB\}$ or, alternatively,
$\BType \NIN \{\tAB, \tO\}$. 
Clearly, this kind of explanation reveals more information about
why the classifier made its decision.
We will later formalize and study a new class of explanations for this purpose, called \textit{general
sufficient reasons,}  which arise only when the classifier has
non-binary features.

A necessary reason for a decision can also be 
understood as a property of the instance, but one that will flip the decision if violated in a \textit{certain} manner~\cite{DBLP:conf/aaai/DarwicheJ22}. 
As mentioned earlier, $\{\eql{\Weight}{\OWeight}, \eql{\BType}{\tA}\}$ is a necessary reason for the decision on Susan. This reason corresponds to the property $(\eql{\Weight}{\OWeight}$ or $\eql{\BType}{\tA})$.
We can flip the decision by violating this property through changing
the values of $\Weight$ and $\BType$ in the instance. 
Since these variables are non-binary, there are multiple changes (six total) that will violate 
the property. Some violations will flip the decision, others will not (we are only guaranteed that at least one violation will flip the decision). 
For example, $\eql{\Weight}{\Nom}, \eql{\BType}{\tO}$ and $\eql{\Weight}{\UWeight}, \eql{\BType}{\tAB}$
will both violate the property but only the first one will flip the decision. 
However, 
the following weaker property is guaranteed to flip the decision regardless of how it is violated: 
$(\eql{\Weight}{\OWeight}$ or $\BType \IN \{\tA, \tB, \tAB\})$. 
We can violate this property using two different settings of $\Weight$ and $\BType$, both of
which will flip the decision. This property corresponds to the
{\em general necessary reason}
$\{\eql{\Weight}{\OWeight}, \BType \IN \{\tA, \tB, \tAB\}\}$, a new notion that
we introduce and study later. 
Similar to general sufficient reasons,
general necessary reasons provide more information about the
behavior of a classifier and arise only when the classifier has non-binary features.

We stress here that using simple explanations in the presence 
of non-binary features is quite prevalent in the 
literature;
see, e.g.,~\cite{DBLP:journals/dke/AudemardBBKLM22,DBLP:journals/corr/abs-2209-07740,DBLP:conf/cikm/BoumazouzaAMT21,DBLP:conf/aaai/DarwicheJ22,DBLP:conf/aaai/IgnatievIS022,DBLP:conf/ijcai/Izza021,DBLP:conf/icml/0001GCIN21}.
Two notable exceptions are \cite{corr/abs-2007-01493,DBLP:journals/jair/IzzaIM22} which we discuss in more detail later.\footnote{Interestingly, the axiomatic study of explanations in~\cite{DBLP:conf/ijcai/AmgoudB22} allows non-binary features, yet Axiom 4 (\textit{feasibility}) implies that explanations must be simple.}

Our study of general necessary and sufficient reasons  
follows a similar structure to recent developments on classical necessary
and sufficient reasons. In particular, we define a new quantification operator
like the one defined in~\cite{jair/DarwicheM21} and show how it can 
be used to compute the \textit{general reason} of a decision, and that its
prime implicates and implicants contain the general necessary and sufficient reasons. Complete reasons are known to be monotone formulas.
We show that 
general reasons are \textit{fixated formulas} which include
monotone ones. We introduce the fixation property and discuss some of its (computational) implications.

This paper is structured as follows.
We start in Section~\ref{sec:discrete} by discussing the syntax and
semantics of formulas with discrete variables which are needed to capture the input-output
behavior of classifiers with non-binary features.
We then introduce the new quantification operator in 
Section~\ref{sec:forall} where we study its properties and show how it
can be used to formulate the new notion of general reason.
The study of general necessary and sufficient reasons is conducted in
Section~\ref{sec:NS reasons} where we also relate them to their classical 
counterparts and argue further for their utility.
Section~\ref{sec:C reason} provides closed-form general reasons for
a broad class of classifiers and Section~\ref{sec:PI-IP} discusses the
computation of general necessary and sufficient reasons based on
general reasons. 
We finally close with some remarks in Section~\ref{sec:conclusion}.
Proofs of all results can be found in Appendix~\ref{sec:proofs}.
 
\section{Representing Classifiers using Class Formulas}
\label{sec:discrete}

We now discuss the syntax and semantics of 
\textit{discrete formulas,} which we use to represent the input-output behavior of classifiers. Such symbolic formulas can be automatically compiled from certain classifiers, like Bayesian networks, random forests and some types of neural networks; see~\cite{lics/Darwiche23} for a summary. 

We assume a finite set of variables \(\Sigma\) which represent 
classifier features. Each variable \(X \in \Sigma\) has a finite number of 
\textit{states} \(x_1, \ldots, x_n,\) \(n > 1\). 
%When \(n=2\), variable \(X\) is binary or Boolean. 
A \textit{literal} \(\l\) for variable \(X\), called \(X\)-literal,
is a set of states 
such that \(\emptyset \subset \l \subset \{x_1, \ldots, x_n\}\). 
We will often denote a literal such as \(\{x_1,x_3,x_4\}\) 
by \(x_{134}\) which reads: the state of variable \(X\) 
is either \(x_1\) or \(x_3\) or \(x_4\).
A literal is \textit{simple} iff it contains
a single state. Hence, \(x_3\) is a simple literal but
\(x_{134}\) is not.
Since a simple literal corresponds to a state, 
these two notions are interchangeable. 

A \textit{formula} is either a constant \(\top\), \(\bot\),
literal \(\lit\), negation \(\NOT \alpha\), conjunction
\(\alpha \AND \beta\) or disjunction \(\alpha \OR \beta\) 
where \(\alpha\), \(\beta\) are formulas. 
The set of variables appearing in a formula $\Delta$
are denoted by $\SetV(\Delta)$.
%It is \textit{Boolean} if it contains only Boolean variables.
A \textit{term} is a conjunction of literals for distinct variables. 
A \textit{clause} is a disjunction of literals for distinct variables.
A \textit{DNF} is a disjunction of terms.
A \textit{CNF} is a conjunction of clauses.
An \textit{NNF} is a formula without negations.  
These definitions imply that terms cannot be inconsistent,
clauses cannot be valid, and negations are not allowed 
in DNFs, CNFs, or NNFs.
Finally, we say a term/clause is \textit{simple} iff 
it contains only simple literals.

%The negation of a simple literal is not simple unless the corresponding variable is Boolean.

A \textit{world} maps each variable in \(\Sigma\) to
one of its states and is typically denoted by \(\w\).
A world \(\w\) is called a \textit{model} of formula \(\alpha\),
written \(\w \models \alpha\), iff \(\alpha\) is satisfied 
by \(\w\) (that is, \(\alpha\) is true at \(\w\)). 
The constant \(\top\) denotes a valid formula (satisfied by every world) 
and the constant \(\bot\) denotes an unsatisfiable formula (has no models). 
Formula \(\alpha\) implies formula \(\beta\), 
written \(\alpha \models \beta\), iff
every model of \(\alpha\) is also a model of \(\beta\).
A term $\tau_1$ subsumes another term $\tau_2$ iff $\tau_2 \models \tau_1$.
A clause $\sigma_1$ subsumes another clause $\sigma_2$ iff $\sigma_1 \models \sigma_2$.
Formula $\alpha$ is weaker than formula $\beta$ iff $\beta \models \alpha$ (hence $\beta$ is stronger than $\alpha$).

The \textit{conditioning} of formula \(\Delta\) on simple term \(\tau\) is denoted \(\Delta | \tau\) and obtained as follows. 
For each state \(x\) of variable \(X\) that appears in term \(\tau\),  
replace each \(X\)-literal \(\l\) in \(\Delta\) with \(\top\) 
if \(x \in \l\) and with \(\bot\) otherwise. Note that \(\Delta | \tau\) does not mention any variable that appears in term $\tau$.
A \textit{prime implicant} for a formula $\Delta$ is a term $\alpha$ such 
that $\alpha \models \Delta$, and there does not exist a distinct 
term $\beta$ such that
\(
\alpha \models \beta \models \Delta. 
\)
A \textit{prime implicate} for a formula $\Delta$ is a clause $\alpha$ 
such that $\Delta \models \alpha$, and there does not exist a distinct 
clause $\beta$ such that
\(
\Delta \models \beta \models \alpha.
\)

An \textit{instance} of a classifier will be represented by
a simple term which contains exactly one literal for each 
variable in \(\Sigma\).
A classifier with \(n\) classes will be represented by a set of mutually 
exclusive and exhaustive formulas \(\Delta^1, \ldots, \Delta^n\), where the models of formula \(\Delta^i\) capture the instances in the \(i^{th}\) class. 
That is, instance \(\instance\) is in the \(i^{th}\) class iff 
\(\instance \models \Delta^i\).
We refer to each \(\Delta^i\) as a \textit{class formula,} or simply a \textit{class,} and 
say that instance \(\instance\) is in class \(\Delta^i\) when \(\instance \models \Delta^i\). 
\begin{wrapfigure}{r}{0.25\textwidth}
    \vspace{-\baselineskip}
        \centering
        \scalebox{0.65}{
        \begin{tikzpicture}[
        roundnode/.style={circle ,draw=black, thick},
        squarednode/.style={rectangle, draw=black, thick},
        ]
        %Nodes
        \node[squarednode]     (X)                              {\Large X};
        \node[roundnode]     (c1)       [below=of X, xshift = -0.8cm, yshift = 0.5cm] {\Large $c_1$};
        \node[squarednode]     (Y)       [below=of X, xshift = 0.8cm, yshift = 0.5cm] {\Large $Y$};
        \node[squarednode]       (Z1)   [below=of Y, xshift = -0.8cm, yshift = 0.5cm] {\Large $Z$};
        \node[squarednode]       (Z2)   [below=of Y, xshift = 0.8cm, yshift = 0.5cm] {\Large $Z$};
        \node[roundnode]       (c2)   [below=of Z1, xshift = -0.8cm, yshift = 0.5cm] {\Large $c_1$};
        \node[roundnode]       (c3)   [below=of Z1, xshift = 0.8cm, yshift = 0.5cm] {\Large $c_2$};
        \node[roundnode]       (c4)   [below=of Z2, xshift = 0.8cm, yshift = 0.5cm] {\Large $c_3$};
        
        %Lines
        \draw[-latex, thick] (X.240) -- node [anchor = center, xshift = -4mm, yshift = 1mm] {$x_1x_2$} (c1.north);
        \draw[-latex, thick] (X.300) -- node [anchor = center, xshift = 4mm, yshift = 1mm] {$x_3$} (Y.north);
        \draw[-latex, thick] (Y.240) -- node [anchor = center, xshift = -4mm, yshift = 1mm] {$y_1$} (Z1.north);
        \draw[-latex, thick] (Y.300) -- node [anchor = center, xshift = 4mm, yshift = 1mm] {$y_2y_3$} (Z2.north);
        \draw[-latex, thick] (Z1.240) -- node [anchor = center, xshift = -4mm, yshift = 1mm] {$z_1z_3$} (c2.north);
        \draw[-latex, thick] (Z1.300) -- node [anchor = center, xshift = -4mm] {$z_2$} (c3.north);
        \draw[-latex, thick] (Z2.240) -- node [anchor = center, xshift = 4mm] {$z_2$} (c3.north);
        \draw[-latex, thick] (Z2.300) -- node [anchor = center, xshift = 4mm, yshift = 1mm] {$z_1z_3$} (c4.north);
        \end{tikzpicture}
        }
    \vspace{-\baselineskip}
        %\caption{A classifier with three classes: $c_1$, $c_2$, and $c_3$. \label{fig:class-formula-dg}}
\end{wrapfigure}
Consider the decision diagram on the right
which represents a classifier with three ternary features (\(X, Y, Z\))
and three classes $c_1$, $c_2$, and $c_3$.
This classifier can be represented by the class formulas
$\Delta^1 = x_{12}\OR x_3 \AND y_1 \AND z_{13}$, 
$\Delta^2 = x_3 \AND z_2$ and 
$\Delta^3 = x_3 \AND y_{23} \AND z_{13}$. 
This classifier has \(27\) instances,
partitioned as follows:
\(20\) instances in class \(c_1\), \(3\) in class \(c_2\)
and \(4\) in class \(c_3\). For example,
instance $\instance = x_3 \AND y_2 \AND z_2$ belongs
to class \(c_2\) since $\instance \models \Delta^2$.

\section{The General Reason for a Decision}
\label{sec:forall}
  
An operator \(\forall x\) which eliminates the state \(x\) of 
a Boolean variable $X$ from a formula was
introduced and studied in~\cite{jair/DarwicheM21}. This operator, called 
universal literal quantification, was also generalized in~\cite{jair/DarwicheM21} to the states
of discrete variables but without further study. Later,
\cite{DBLP:conf/aaai/DarwicheJ22} studied this discrete
generalization, given next.

\begin{definition}\label{def:forall1}
For variable \(X\) with states \(x_1, \ldots, x_n\),
the universal literal quantification of state $x_i$ from formula $\Delta$ is defined as $\forall x_i \cdot \Delta = \Delta | x_i \AND \BAND_{j \not = i} (x_i \OR \Delta | x_j)$.
\end{definition}

The operator $\forall$ is commutative so we can equivalently 
write \(\forall x \cdot (\forall y \cdot \Delta)\), 
\(\forall y \cdot (\forall x \cdot \Delta)\),
\(\forall x,y \cdot \Delta\) or \(\forall \{x,y\} \cdot \Delta\).
It is meaningful then to quantify an instance \(\instance\) from its
class formula \(\Delta\) 
since \(\instance\) is a set of states.
As shown in~\cite{jair/DarwicheM21},
the quantified formula \(\forall \instance \cdot \Delta\) corresponds to the complete reason for the
decision on instance \(\instance\).
Hence, the prime implicants of \(\forall \instance \cdot \Delta\)
are the sufficient reasons for the decision~\cite{ecai/DarwicheH20}
and its prime implicates are the necessary reasons~\cite{DBLP:conf/aaai/DarwicheJ22}.

We next define a new operator $\iuq$ that we
call a \textit{selection operator} for reasons that will become
apparent later. This operator will lead to the
notion of a general reason for a decision which subsumes
the decision's complete reason, and provides the basis
for defining general necessary and sufficient reasons.

\begin{definition}\label{def:forall2}
For variable \(X\) with states \(x_1, \ldots, x_n\) and formula \(\Delta\),
we define $\iuq x_i \cdot \Delta$ to be
$\Delta | x_i \AND \Delta$.
\end{definition}

The selection operator \(\iuq\) is also commutative, like \(\forall\). 
\begin{proposition}\label{prop:commutative}
$\iuq x \cdot (\iuq y \cdot \Delta) = \iuq y \cdot (\iuq x \cdot \Delta)$
for states \(x, y\).
\end{proposition}
Since a term \(\tau\) corresponds to a set of states, the expression
\(\iuq \tau \cdot \Delta\) is 
well-defined just like \(\forall \tau \cdot \Delta\).
We can now define our first major notion.

\begin{definition} \label{def:gcr}
Let \(\instance\) be an instance in class \(\Delta\).
The general reason for the decision on instance \(\instance\)
is defined as \(\iuq \instance \cdot \Delta\).
\end{definition}

The complete reason \(\forall \instance \cdot \Delta\) can be 
thought of as a property/abstraction of instance \(\instance\) that justifies 
(i.e., can trigger) the decision. In fact, it is equivalent to
the weakest NNF \(\Gamma\) whose literals appear in the instance
and that satisfies \(\instance \models \Gamma \models \Delta\)~\cite{jair/DarwicheM21,DBLP:conf/aaai/DarwicheJ22}. 
The next result shows that
the general reason is a weaker
property and, hence, a further abstraction that triggers 
the decision.

\begin{proposition}\label{prop:forall-relation}
For instance $\instance$ and formula $\Delta$ 
where $\instance \models \Delta$, we have 
$\instance \models \forall \instance \cdot \Delta \models \iuq \instance \cdot \Delta \models \Delta$. (\(\instance \not \models \Delta\)
only if $\forall \instance \cdot \Delta =\iuq \instance \cdot \Delta = \bot$)
\end{proposition}

The next result provides further semantics for the general
reason and highlights the key difference with the 
complete reason.

\begin{proposition}\label{prop:gr semantics}
The general reason $\iuq \instance \cdot \Delta$ is
equivalent to the weakest NNF \(\Gamma\) whose literals are
implied by instance \(\instance\) and that satisfies 
\(\instance \models \Gamma \models \Delta\).
\end{proposition}

The complete and general reasons are abstractions of
the instance that explain why it belongs to its class. The former
can only reference simple literals in the instance but
the latter can reference any literal that is implied by the instance.
The complete reason can be recovered from the general reason
and the underlying instance. Moreover, the two types of reasons are equivalent when all variables are binary since
\(\forall x \cdot \Delta = \iuq x \cdot \Delta\) when \(x\)
is the state of a binary variable. 

We next provide a number of results that further our understanding of general reasons, particularly their
semantics and how to compute them.
We start with the following alternative definition of the operator $\iuq x_i$.

\begin{proposition}\label{prop:alternative_gr}
For formula $\Delta$ and variable $X$ with states $x_1, \ldots, x_n$, 
$\iuq x_i \cdot \Delta$ is equivalent to $(\Delta | x_i) \AND \BAND_{j \not = i}( \ell_{j} \OR (\Delta | x_j))$,
where $\ell_{j}$ is the literal $\{x_1, \ldots, x_n\}\setminus \{x_j\}$.
\end{proposition}
According to this definition, we can always express $\iuq x_i \cdot \Delta$ 
as an NNF in which every $X$-literal includes state $x_i$ (recall that
$\Delta | x_i$ and $\Delta | x_j$ do not mention variable $X$). This
property is used in the proofs and has a number of implications.\footnote{For example, we can use it to
provide \textit{forgetting} semantics for the dual operator
$\ieq x_i \cdot \Delta = \NOT{\iuq x_i \cdot \NOT{\Delta}}$.
Using Definition~\ref{def:forall2}, we get 
$\ieq x_i \cdot \Delta = \Delta \OR \Delta | x_i$.
Using Proposition~\ref{prop:alternative_gr}, we get
$\ieq x_i \cdot \Delta =
\Delta | x_i \OR \BOR_{j \not = i} (x_j \AND \Delta | x_j)$.
We can now easily show that 
(1)~$\Delta \models \ieq x_i \cdot \Delta$ and 
(2)~$\ieq x_i \cdot \Delta$ is equivalent to an NNF 
whose $X$-literals do not mention state $x_i$. 
That is, $\ieq x_i$ can be understood as forgetting the information
about state $x_i$ from $\Delta$.
This is similar to the dual operator $\exists x_i \cdot \Delta
= \NOT{\forall x_i \cdot \NOT{\Delta}}$ studied in~\cite{jair/LangLM03,jair/DarwicheM21} except that $\ieq x_i$ erases less information from $\Delta$ since one can show that
$\Delta \models \ieq x_i \cdot \Delta \models \exists x_i \cdot \Delta$.
}

When \(\Delta\) is a class formula, 
\cite{jair/DarwicheM21} showed that the application of \(\forall x\)
to \(\Delta\) can be 
understood as \textit{selecting} a specific set of instances from the corresponding
class. This was shown for states \(x\) of Boolean variables. We next
generalize this to discrete variables and
provide a selection semantics for the new operator \(\iuq\).

\begin{proposition}
\label{prop:selection semantics}
Let \(\tau\) be a simple term, $\Delta$ be a formula
and \(\w\) be a world.
Then
$\w \models \forall \tau \cdot \Delta$ 
iff $\w \models \Delta$ and
$\w' \models \Delta$ for any world \(\w'\) obtained from \(\w\)
by changing the states of some variables that are set differently in \(\tau\). 
Moreover, $\w \models \iuq \tau \cdot \Delta$ iff 
$\w \models \Delta$ and \(\w' \models \Delta\) for any world
\(\w'\) obtained from \(\w\) by setting some variables in \(\w\) to 
their states in \(\tau\). 
\end{proposition}

That is, $\forall \tau \cdot \Delta$ selects all instances in class
\(\Delta\) whose membership in the class does not depend on 
characteristics that are inconsistent with \(\tau\).
These instances are also selected by $\iuq \tau \cdot \Delta$ which
further selects instances that remain in class $\Delta$ when any of their characteristics are changed to agree with $\tau$. 

The complete reason is monotone which has
key computational implications as shown in~\cite{ecai/DarwicheH20,jair/DarwicheM21,DBLP:conf/aaai/DarwicheJ22}.
The general reason satisfies 
a weaker property called \textit{fixation} which has also
key computational implications as we show in Section~\ref{sec:PI-IP}.

\begin{definition}\label{def:fixation}
An NNF is locally fixated on instance $\instance$  
iff its literals are consistent with~$\instance$.
A formula is fixated on instance $\instance$ iff it is
equivalent to an NNF that is locally fixated on $\instance$.
\end{definition}
We also say in this case that the formula is \(\instance\)-fixated. 
For example, if $\instance = x_1\AND y_1 \AND z_2$ then 
the formula $x_{12} \AND y_1 \OR z_2$ is (locally) $\instance$-fixated 
but $x_{12} \AND z_1$ is not.
By the selection semantic we discussed earlier, 
a formula $\Delta$ is $\instance$-fixated only if for every model $\w$ of $\Delta$, changing the states of some variables in $\w$ to their states
in $\instance$ guarantees that the result remains a model of $\Delta$.
Moreover, if \(\Delta\) is \(\instance\)-fixated, then
\(\instance \models \Delta\) but the opposite does not hold (e.g.,
$\Delta = x_1 \OR y_1$ and $\instance = x_1 \AND y_2$).
We now have the following corollary of Proposition~\ref{prop:gr semantics}.

\begin{corollary}
The general reason $\iuq \instance \cdot \Delta$ is
$\instance$-fixated.
\end{corollary}

The next propositions show that the new operator \(\iuq\)
has similar computational properties to \(\forall\) which we
use in Section~\ref{sec:C reason} to compute general reasons.

\begin{proposition}\label{prop:quantify-b}
For state $x$ and literal $\lit$ of variable $X$,  
$\iuq x \cdot \lit = \lit$ if $x \in \lit$ ($x \models \lit$); 
else $\iuq x \cdot \lit = \bot$. Moreover,
$\iuq x \cdot \Delta = \Delta$ if $X$ does not appear in $\Delta$.
\end{proposition}

\begin{proposition}\label{prop:distribute-and-or}
For formulas \(\alpha\), \(\beta\) and state \(x_i\) of variable \(X\), we have
\(\iuq x_i \cdot (\alpha \AND \beta)  = (\iuq x_i \cdot \alpha) \AND (\iuq x_i \cdot \beta)\).
Moreover, if variable \(X\) does not occur in both \(\alpha\) and \(\beta\), then
 \(\iuq x_i \cdot (\alpha \OR \beta) = (\iuq x_i \cdot \alpha) \OR (\iuq x_i \cdot \beta)\).
\end{proposition}
An NNF is \(\vee\)-decomposable if its disjuncts do not share variables. According to these propositions, we can 
apply \(\iuq \instance\) to an \(\vee\)-decomposable NNF in linear time,
by simply applying \(\iuq \instance\) to each literal in the NNF (the result is \(\vee\)-decomposable).

\section{General Necessary and Sufficient Reasons}
\label{sec:NS reasons}

We next introduce generalizations of necessary and sufficient
reasons and show that they are prime implicates and implicants
of the general reason for a decision. These new notions have more explanatory power and subsume
their classical counterparts, particularly when explaining
the behavior of a classifier beyond a specific instance/decision. 
For example, when considering the classifier
in Figure~\ref{fig:alternative-dg}, 
which is a variant of the one in Figure~\ref{fig:disease-dg}, we will
see that the two classifiers will make identical decisions on
some instances, leading to identical simple necessary and sufficient
reasons for these decisions but distinct general necessary and
sufficient reasons. Moreover, we will see that general necessary
and sufficient reasons are particularly critical when explaining
the behavior of classifiers with (discretized) numeric features.

\subsection{General Sufficient Reasons (GSRs)}

We start by defining the classical notion of a (simple) 
sufficient reason 
but using a different formulation than~\cite{ijcai/ShihCD18}
which was the first to introduce this notion under the name of a PI-explanation.
Our formulation is meant to highlight a symmetry with the proposed
generalization.

\begin{definition}[SR] \label{def:sr}
A sufficient reason for the decision on instance \(\instance\) in
class \(\Delta\) is a weakest simple term $\tau$ 
s.t. $\instance \models \tau \models \Delta$.
\end{definition}
This definition implies that each literal in \(\tau\)
is a variable setting (i.e., characteristic) that appears in instance \(\instance\).
That is, the (simple) literals of sufficient reason \(\tau\) are
a subset of the literals in instance \(\instance\).
We now define our generalization.

\begin{definition}[GSR]\label{def:gsr}
A general sufficient reason for the decision on instance \(\instance\) in
class \(\Delta\) is a term \(\tau\) which satisfies
(1)~\(\tau\) is a weakest term s.t. $\instance \models \tau \models \Delta$ and (2)~no term $\tau'$ satisfies the previous condition if $\SetV(\tau') \subset \SetV(\tau)$.
\end{definition}
This definition does not require the GSR \(\tau\) to be a simple term, but it requires that it has  
a minimal set of variables.
Without this minimality condition, a GSR will be redundant in the sense of the upcoming 
Proposition~\ref{prop:redundant sr}. 
For a term \(\tau\) and instance \(\instance\) s.t. \(\instance \models \tau\), we will use 
\(\instance \mycap \tau\) to denote the smallest subterm
in \(\instance\) that implies \(\tau\). For example,
if \(\instance = x_2 \AND y_1 \AND z_3\) and \(\tau = x_{12} \AND y_{13}\), 
then \(\instance \mycap \tau = x_2 \AND y_1\). 

\begin{proposition}
\label{prop:redundant sr}
Let $\instance$ be an instance in class $\Delta$
and \(\tau\) be a weakest term s.t. 
$\instance \models \tau \models \Delta$.
If $\tau'$ is a weakest term  s.t. $\instance \models \tau' \models \Delta$ and $\SetV(\tau') \subset \SetV(\tau)$,
then $\instance \mycap \tau \models \instance \mycap \tau' \models \Delta$. Also, $\instance \mycap \tau$ is a SR iff such a term $\tau'$ does not exist.
\end{proposition}
According to this proposition, the term \(\tau\) is redundant
as an explanation
in that
the subset of instance \(\instance\) which it identifies as being
a culprit for the decision (\(\instance \mycap \tau\)) is dominated by a smaller subset that is identified by the 
term \(\tau'\) ($\instance \mycap \tau'$).

Consider the classifiers in Figures~\ref{fig:disease-dg} and~\ref{fig:alternative-dg} and the patient Susan:
$\Age \GE 55$, $\eql{\BType}{\tA}$ and
$\eql{\Weight}{\OWeight}$.
Both classifiers will make the same decision $\yes$ on Susan
with the same SRs:
$(\Age \GE 55 \AND \eql{\BType}{\tA})$ and  $(\Age \GE 55 \AND \eql{\Weight}{\OWeight})$. The GSRs are different
for these two (equal) decisions. For the first classifier,
they are 
$(\Age\GE 55 \AND \BType \IN \{\tA, \tB\})$ and $(\Age \GE 55 \AND \eql{\Weight}{\OWeight})$.
For the second, they are
$(\Age\GE 55 \AND \BType \IN \{\tA, \tO\})$ and $(\Age\GE 55 \AND \Weight \IN \{\OWeight, \Nom\})$.
GSRs
encode all SRs and
contain more information.\footnote{Unlike SRs, two GSRs may mention the same set of variables. Consider the class formula $\Delta = (x_1 \AND y_{12}) \OR (x_{12} \AND y_1)$ and instance $\instance = x_1 \AND y_1$. There are two GSRs for the decision on $\instance$, $x_1 \AND y_{12}$ and $x_{12} \AND y_1$, and both mention the same variables $X,Y$.}

\begin{proposition}\label{prop:gsr_no_lost}
Let $\tau$ be a simple term. Then $\tau$ is a SR for the 
decision on instance \(\instance\) iff $\tau = \instance \mycap \tau'$ for some GSR $\tau'$.
\end{proposition}
Consider the instance Susan again, $\instance = (\Age \GE 55) \AND (\eql{\BType}{\tA}) \AND (\eql{\Weight}{\OWeight})$ and the classifier in
Figure~\ref{fig:alternative-dg}. As mentioned, 
the GSRs for the decision on Susan are
$\tau'_1=(\Age\GE 55 \AND \BType \IN \{\tA, \tO\})$ 
and 
$\tau'_2=(\Age\GE 55 \AND \Weight \IN \{\OWeight, \Nom\})$
so $\tau_1 = \instance \mycap \tau'_1 = (\Age \GE 55 \AND \eql{\BType}{\tA})$
and 
$\tau_2 = \instance \mycap \tau'_2 = (\Age \GE 55 \AND \eql{\Weight}{\OWeight})$,
which are the two SRs for the decision on Susan.

The use of general terms to explain the decision on an
instance \(\instance\) in class \(\Delta\) was first suggested in~\cite{corr/abs-2007-01493}. This work proposed the notion of a general PI-explanation as a prime implicant of 
\(\Delta\) that is consistent with instance \(\instance\).
This definition is equivalent to Condition~(1) in
our Definition~\ref{def:gsr} which has a second condition
relating to variable minimality. Hence, the definition
proposed by~\cite{corr/abs-2007-01493} does not satisfy
the desirable properties stated in Propositions~\ref{prop:redundant sr}
and~\ref{prop:gsr_no_lost} which require this minimality condition. 
The merits of using general terms were also discussed when explaining decision trees in~\cite{DBLP:journals/jair/IzzaIM22}, which introduced the notion of an \textit{abductive path explanation (APXp).} 
In a nutshell, each path in a decision tree
corresponds to a general term \(\tau\) that implies the
formula \(\Delta\) of the path's class. Such a term is usually used to explain the decisions made on instances that follow that path. As observed in~\cite{DBLP:journals/jair/IzzaIM22}, such a term can often be shortened, leading
to an APXp that still implies the class formula \(\Delta\) and hence provides a better explanation. 
An APXp is an implicant of the class formula $\Delta$ but not necessarily a prime implicant (or a variable-minimal prime implicant). Moreover, an APXp is a property of the specific decision tree (syntax) instead of its underlying classifier (semantics). See Appendix~\ref{sec:path exp diff} for further discussion of these limitations.\footnote{A dual notion, contrastive path explanation (CPXp), was also proposed in~\cite{DBLP:journals/jair/IzzaIM22}.}

\subsection{General Necessary Reasons (GNRs)}

We now turn to simple necessary reasons and their
generalizations. A necessary reason is a property of the instance
that will flip the decision if violated in a certain way (by changing the instance).
As mentioned earlier, the difference
between the classical necessary reason and the generalized one
is that the latter comes with stronger guarantees. 
Again, we start with a definition of classical necessary reasons
using a different phrasing than~\cite{aiia/IgnatievNA020} which formalized them
under the name of contrastive explanations~\cite{lipton_1990}. Our phrasing, based on~\cite{DBLP:conf/aaai/DarwicheJ22},
highlights a symmetry with the generalization and requires the following
notation.

For a clause \(\sigma\) and instance \(\instance\) s.t. \(\instance \models \sigma\),
we will use \(\instance \myminus \sigma\) to denote the largest
subterm of \(\instance\) that does not imply \(\sigma\). For example,
if \(\instance = x_2 \AND y_1 \AND z_3\) and 
\(\sigma = x_{12} \OR y_{13}\) then
\(\instance \myminus \sigma = z_3\).
We will also write \(\instance \mymodels \sigma\) to mean that
instance \(\instance\) implies every literal in clause \(\sigma\).
For instance
\(\instance = x_2 \AND y_1 \AND z_3\), we have \(\instance \mymodels x_{12} \OR y_{13}\)
but \(\instance \notmymodels x_{12} \OR y_{23}\) even though \(\instance \models x_{12} \OR y_{23}\).

\begin{definition}[NR]
\label{def:nr}
A necessary reason for the decision on instance \(\instance\) in class 
\(\Delta\) is a strongest simple clause 
$\sigma$ s.t. 
\(\instance \mymodels \sigma\) and
$(\instance \myminus \sigma) \AND \NOT{\sigma} \not \models \Delta$
(if we minimally change the instance to violate \(\sigma\), it is no longer
guaranteed to stay in class \(\Delta\)).
\end{definition}

A necessary reason guarantees that \textit{some}
minimal change to the instance which violates the reason will flip the decision.
But it does not guarantee that \textit{all} such changes will. 
A general necessary reason comes with a stronger guarantee.

\begin{definition}[GNR]
\label{def:gnr}
A general necessary reason for the decision on instance \(\instance\) in class 
\(\Delta\) is a strongest clause $\sigma$ s.t.
$\instance \mymodels \sigma$,
$(\instance \myminus \sigma) \AND \NOT{\sigma} \models \NOT{\Delta}$, and
no clause $\sigma'$ satisfies the previous conditions if $\SetV(\sigma') \subset \SetV(\sigma)$.
\end{definition}
The key difference between Definitions~\ref{def:nr} and~\ref{def:gnr} are the conditions
\((\instance \myminus \sigma) \AND \NOT{\sigma} \not \models \Delta\) and \((\instance \myminus \sigma) \AND \NOT{\sigma} \models \NOT{\Delta}\). The first condition
guarantees that \textit{some} violation of a NR will flip the decision
(by placing the modified instance outside class \(\Delta\))
while the second condition guarantees that \textit{all} violations of a GNR
will flip the decision.

The next proposition explains why we require GNRs to be
variable-minimal. Without this condition, the changes identified by a GNR to flip the decision may not be 
minimal (we 
can flip the decision by changing a strict subset of variables).

For instance $\instance$ and clause $\sigma$ s.t. $\instance \models \sigma$, we will use $\instance \mycap \sigma$ to denote
the disjunction of states that appear in both $\instance$ and $\sigma$ (hence, $\instance \mycap \sigma \models \sigma$).
For example, if $\instance = x_1 \AND y_1 \AND z_1$ and $\sigma = x_{12} \OR y_{23} \OR z_1$, then $\instance \mycap \sigma = x_1 \OR z_1$.

\begin{proposition}
\label{prop:redundant nr}
Let $\instance$ be an instance in class $\Delta$
and let \(\sigma\) be a strongest clause s.t. 
$\instance \mymodels \sigma$ and $(\instance \myminus \sigma) \AND \NOT \sigma \models \NOT{\Delta}$. 
If $\sigma'$ is 
another strongest clause satisfying these conditions
and $\SetV(\sigma') \subset \SetV(\sigma)$,
then $\instance \myminus \sigma' \models \instance \myminus \sigma$. Moreover, $\instance \mycap \sigma$ is a NR iff such a clause $\sigma'$ does not exist.
\end{proposition}
That is, if violating \(\sigma\) requires changing
some characteristics $C$ of instance \(\instance\), then \(\sigma'\) can be violated by changing a strict subset of these characteristics $C$.

Consider the classifiers in Figures~\ref{fig:disease-dg} and~\ref{fig:alternative-dg} which make the same decision, $\yes$, on
Susan ($\Age \GE 55$, $\eql{\BType}{\tA}$, $\eql{\Weight}{\OWeight}$).
The NRs for these equal decisions 
are the same: 
$(\Age \GE 55)$ and 
$(\eql{\Weight}{\OWeight} \OR \eql{\BType}{\tA})$.
The GNRs for the classifier in 
Figure~\ref{fig:disease-dg} are 
$(\Age\GE 55)$, 
$(\BType \IN \{\tA, \tB, \tAB\} \OR$  $\eql{\Weight}{\OWeight}\})$ and
$(\BType \IN \{\tA, \tB\} \OR \Weight\IN\{\UWeight, \OWeight\})$.
If the instance is changed to violate any of them, the decision will change. 
For example, if we set $\BType$ to $\tAB$ and $\Weight$ to $\Nom$, the third GNR will be violated and the decision on Susan becomes $\no$.
For the classifier in Figure~\ref{fig:alternative-dg},
the GNRs for the decision are 
different:
$(\Age \GE 55)$ and $(\BType \IN \{\tA, \tO\} \OR \Weight \IN \{\Nom, \OWeight\})$. 
However, both sets of GNRs contain
more information than the NRs since 
the minimal changes they identify
to flip the decision include those identified
by the NRs.

\begin{proposition}\label{prop:gnr_no_lost}
Let $\sigma$ be a simple clause.
Then $\sigma$ is a NR for the 
decision on instance \(\instance\) iff $\sigma = \instance \mycap \sigma'$ for some GNR $\sigma'$.
\end{proposition}
Consider the instance Susan again, $\instance = (\Age \GE 55) \AND (\eql{\BType}{\tA}) \AND (\eql{\Weight}{\OWeight})$ and the classifier in
Figure~\ref{fig:alternative-dg}. As mentioned earlier, 
the GNRs for the decision on Susan are
$\sigma'_1 = (\Age \GE 55)$ and $\sigma'_2 = (\BType \IN \{\tA, \tO\} \OR \Weight \IN \{\Nom, \OWeight\})$. Then
$\sigma_1 =\instance \mycap \sigma'_1 = (\Age \GE 55)$
and 
$\sigma_2 = \instance \mycap \sigma'_2 = (\eql{\Weight}{\OWeight} \OR \eql{\BType}{\tA})$, which are the two NRs for the decision on Susan.

GSRs and GNRs are particularly significant when
explaining the decisions of classifiers with numeric features, a topic which we discuss in Appendix~\ref{sec:numeric}.

We next present a fundamental result which allows us to compute GSRs and GNRs using
the general reason for a decision (we use this
result in Section~\ref{sec:PI-IP}). 

\begin{definition}\label{def:var-min}
A prime implicant/implicate $c$ of formula $\Delta$ is 
variable-minimal 
iff 
there is no prime implicant/implicate 
$c'$ of $\Delta$ s.t. $\SetV(c') \subset \SetV(c)$.
\end{definition}

\begin{proposition}\label{prop:PI-IP}
Let \(\instance\) by an instance in class \(\Delta\).
The GSRs/GNRs for the decision on instance
\(\instance\) are the variable-minimal prime implicants/implicates 
of the general reason $\iuq \instance \cdot \Delta$.
\end{proposition}

The disjunction of SRs is
equivalent to the complete reason which is equivalent to the conjunction of NRs. However, the disjunction of GSRs implies the general reason but is not equivalent to it, and the conjunction of GNRs is implied by the general reason but is not equivalent to it; see Appendix~\ref{sec:var-min}. This suggests that more information can potentially be extracted from the
general reason beyond the information provided by GSRs and GNRs.

\section{The General Reasons of Decision Graphs}
\label{sec:C reason}

Decision graphs are DAGs which include decision trees~\cite{10.2307/2985543,DBLP:books/wa/BreimanFOS84}, OBDDs~\cite{tc/Bryant86},
and can have discrete or numeric features. 
They received significant attention in the work
on explainable AI since they
can be compiled from other types of classifiers such as Bayesian 
networks~\cite{aaai/ShihCD19}, random forests~\cite{corr/abs-2007-01493} and some types of neural networks~\cite{kr/ShiSDC20}. Hence, the ability to explain decision graphs has a
direct application to explaining the decisions of a broad class of classifiers.
Moreover, 
the decisions undertaken by decision graphs have closed-form complete reasons as shown in~\cite{DBLP:conf/aaai/DarwicheJ22}.
We provide similar closed forms for the general
reasons in this section.
We first review decision graphs
to formally state our results.

Each leaf node in a decision graph is labeled with some class~\(c\). 
An internal node \(T\) that {\em tests} variable \(X\) has outgoing edges \(\DE(X,S_1,T_1)\), \(\ldots, \DE(X,S_n,T_n)\), 
\(n \geq 2\). 
The children of node \(T\) are \(T_1, \ldots, T_n\)
and \(S_1, \ldots, S_n\) is a partition of {\em some} states of variable \(X\). 
A decision graph will be represented by its root node.
Hence, each node in the graph represents a smaller decision graph.
Variables can be tested more than once on a path if they satisfy the 
{\em weak test-once property} discussed next~\cite{DBLP:conf/aaai/DarwicheJ22,DBLP:conf/kr/HuangII021}. 
Consider a path \(\ldots, T\DE(X,S_j,T_j), \ldots, T'\DE(X,R_k,T_k), \ldots\) from the root to a leaf (nodes
\(T\) and \(T'\) test~\(X\)). 
If no nodes between \(T\)
and \(T'\) on the path test variable \(X\), then \(\{R_k\}_k\) must be a partition of states \(S_j\). Moreover,
if \(T\) is the first node that tests \(X\) on the path, then \(\{S_j\}_j\)
must be a partition of {\em all} states for \(X\).
Discretized numeric variables are normally tested more than
once while satisfying the weak test-once property;
see Appendix~\ref{sec:numeric} for an illustration.

\begin{proposition}\label{prop:closed form}
Let $T$ be a decision graph, $\instance$ be an instance in class $c$, and $\instance[X]$ be the state of variable $X$ in instance~$\instance$.
Suppose $\Delta^c[T]$ is the class formula of $T$ and class $c$. 
The general reason $\iuq \instance \cdot \Delta^c[T]$ is given by the NNF circuit:\footnote{An NNF circuit is a DAG whose leaves are labeled with $\bot, \top$, or literals; and whose internal nodes are labelled with $\AND$ or $\OR$.}
\[
\Gamma^c[T] = \begin{cases}
\top & \text{if $T$ is a leaf with class $c$}\\
\bot & \text{if $T$ is a leaf with class $c' \not = c$}\\
\BAND_j(\Gamma^c[T_j] \OR \l) & \text{if $T$ has outgoing edges $ \xrightarrow{X, S_j} T_j$}\\
\end{cases}
\]
Here, $\l$ is the \(X\)-literal
$\{x_i \mid x_i \not \in S_j\}$ if $\instance[X] \not \in S_j$,
else $\l=\bot$.
\end{proposition}

The following proposition identifies some
properties of the above closed form, which
have key computational implications that
we exploit in the next section. 

\begin{proposition}\label{prop:gr properties}
The NNF circuit in Proposition~\ref{prop:closed form} is locally fixated on instance $\instance$.
Moreover, every disjunction in this circuit has the form 
\(\l \OR \Delta\) where \(\l\) is an $X$-literal, and 
for every $X$-literal $\l'$ in \(\Delta\) we have
$\l' \neq \l$ and $\l \models \l'$.
\end{proposition}

\section{Computing Prime Implicants \& Implicates}
\label{sec:PI-IP}

Computing the prime implicants/implicates of 
Boolean formulas was studied extensively for decades; see,~e.g.,~\cite{1671510,10.1145/131214.131223,kean1990incremental}. 
The classical methods are based on
\textit{resolution} when computing the prime implicates of CNFs,
and \textit{consensus} when computing the prime implicants
of DNFs; see, e.g.,~\cite{gurvich1999generating,Crama2011BooleanF}. More modern approaches are based on passing encodings to SAT-solvers; see, e.g.,~\cite{DBLP:conf/ijcai/PrevitiIMM15,9643520,DBLP:conf/ijcai/Izza021}.
In contrast, the computation of prime implicants/implicates
of discrete formulas has received very little attention 
in the literature. One recent exception is~\cite{corr/abs-2007-01493} which showed how an algorithm for computing prime implicants of Boolean formulas can be used to compute simple prime implicants of discrete formulas given an appropriate encoding.
Computing prime implicants/implicates of NNFs also received relatively little attention; 
see~\cite{ramesh1997cnf,DBLP:conf/aaai/DarwicheJ22,ijcai/ColnetM22} for some exceptions. 
We next provide methods for computing variable-minimal prime implicants/implicates 
of some classes of discrete formulas that are relevant to GSRs and GNRs.

A set of terms \(S\) will be interpreted as a 
DNF \(\sum_{\tau \in S} \tau\)
and a set of clauses \(S\) will be interpreted as a CNF
\(\prod_{\sigma \in S} \sigma\).
If $S_1$ and $S_2$ are two sets of terms, then
\(
S_1 \times S_2 = 
\{\tau_1 \AND \tau_2 \mid \tau_1 \in S_1, \tau_2 \in S_2\}.
\)
For a set of terms/clauses $S$, $\subsum{S}$ denotes
the result of removing subsumed terms/clauses from $S$.

\def\ALGPI{2}
\def\CACHE{\text{CACHE}}

\begin{algorithm}[tb]
\begin{footnotesize}
\caption{$\GSR(\Delta)$ --- without Line~\ref{ln:GSR-vmin}, this is \textbf{Algorithm~\ALGPI} $\PI(\Delta)$
\label{alg:PI inc var-min}}
\begin{algorithmic}[1]
\Require NNF circuit $\Delta$ which satisfies the properties in Proposition~\ref{prop:gr properties}
\If{$\CACHE(\Delta) \not = \text{NIL}$}
\Return $\CACHE(\Delta)$
\ElsIf{$\Delta = \top$}
\Return $\{\top\}$
\ElsIf{$\Delta = \bot$}
\Return $\emptyset$
\ElsIf{$\Delta$ is a literal}
\Return $\{\Delta\}$
\ElsIf{$\Delta = \alpha \AND \beta$}
\State $S \gets \subsum{\GSR(\alpha) \times \GSR(\beta)}$ \label{ln:GSR-sub-and}
\ElsIf{$\Delta = \alpha \OR \beta$}
\State $S \gets \subsum{\GSR(\alpha) \cup \GSR(\beta)}$ \label{ln:GSR-sub-or}
\EndIf 
\State $S \gets \PIIVM{S}{\ivars{\Delta}}$ \label{ln:GSR-vmin}
\State $\CACHE(\Delta) \gets S$
\State \Return $S$
\end{algorithmic}
\end{footnotesize}
\end{algorithm}

\subsection{Computing General Sufficient Reasons}

Our first result is Algorithm~\ref{alg:PI inc var-min} which computes the variable-minimal prime
implicants of an NNF circuit that satisfies the properties
in Proposition~\ref{prop:gr properties} and, hence, is applicable to the 
general reasons of Proposition~\ref{prop:closed form}. If we remove Line~\ref{ln:GSR-vmin} from Algorithm~\ref{alg:PI inc var-min}, it becomes Algorithm~\ALGPI\ which computes all prime implicants instead of only the variable-minimal ones. Algorithm~\ALGPI\ is the 
same algorithm used to convert an NNF into a DNF (i.e., no consensus is invoked), yet 
the resulting DNF is guaranteed to be in prime-implicant
form. Algorithm~\ALGPI\ is justified by the following two results, where the 
first result generalizes Proposition~40 in \cite{marquis2000consequence}.

In the next propositions, $\sPI(\Delta)$
denotes the prime implicants of formula $\Delta$.
\begin{proposition}\label{prop:PI_AND}
$\sPI(\alpha \AND \beta) = 
\subsum{\sPI(\alpha) \times \sPI(\beta)}.$
\end{proposition}

\begin{proposition}\label{prop:PI_OR}
For any disjunction $\alpha \OR \beta$ that satisfies the 
property of Proposition~\ref{prop:gr properties},
$\sPI(\alpha \OR \beta) = \subsum{\sPI(\alpha) \cup \sPI(\beta)}$.
\end{proposition}

We will next explain 
Line~\ref{ln:GSR-vmin} of Algorithm~\ref{alg:PI inc var-min}, $S \gets \PIIVM{S}{\ivars{\Delta}}$, which is responsible for
pruning prime implicants that are not variable-minimal (hence, computing GSRs).
Here, \(\Delta\) is a node in the NNF circuit passed in the 
first call to Algorithm~\ref{alg:PI inc var-min},
and \(\ivars{\Delta}\) denotes
variables that appear only in the sub-circuit rooted
at node \(\Delta\). 
Moreover, $\PIIVM{S}{V}$ is the set of terms
obtained from terms $S$ by removing every term $\tau \in S$ that satisfies $\SetV(\tau) \supset \SetV(\tau')$ and $V \cap (\SetV(\tau) \setminus \SetV(\tau')) \neq \emptyset$
for some other term $\tau' \in S$.\footnote{The condition $V \cap (\SetV(\tau) \setminus \SetV(\tau')) \neq \emptyset$ is trivially satisfied when $\Delta$ is the root of the NNF circuit since $V$ will include all circuit variables in
this case.}
That is, term $\tau$ will be removed only if
some variable $X$ in $\vars(\tau)\setminus \vars(\tau')$ 
appears only in the sub-circuit rooted at node
$\Delta$ (this ensures that term $\tau$ will not participate in constructing
any variable-minimal prime implicant).
This incremental pruning technique is
enabled by the local fixation
property (Definition~\ref{def:fixation}).

\begin{proposition}\label{prop:PI inc-var-min}
Algorithm~\ref{alg:PI inc var-min}, $\GSR(\Delta)$, returns the 
variable-minimal prime implicants of NNF circuit $\Delta$.
\end{proposition}

\subsection{Computing General Necessary Reasons}

We can convert an NNF circuit into a CNF using a dual
of Algorithm~\ALGPI\ but the result will not be
in prime-implicate form, even for ciruits  that satisfy 
the properties 
Proposition~\ref{prop:gr properties}.\footnote{The number of clauses in this CNF will be no more than the number of NNF nodes if the NNF is the general reason
of a decision tree (i.e., the NNF has a tree structure).}
Hence,
we next propose a generalization of the Boolean resolution
inference rule to discrete variables, which can be
used to convert a CNF into its prime-implicate 
form.
Recall first that Boolean resolution 
derives the clause
\(\alpha + \beta\) from the clauses \(x \OR \alpha\)
and \(\NOT{x} + \beta\) where \(X\) is a Boolean variable.

\begin{definition}\label{def:resolution}
Let $\alpha = \l_1 \OR \sigma_1$, 
$\beta = \l_2 \OR \sigma_2$ be two clauses
where $\l_1$ and $\l_2$ are \(X\)-literals s.t. $\l_1 \not \models \l_2$ and $\l_2 \not \models \l_1$.
If \(\sigma = (\l_1 \AND \l_2) \OR \sigma_1 \OR \sigma_2 \neq \top\), then
the \(X\)-resolvent of clauses $\alpha$ and $\beta$ is defined as the clause equivalent to \(\sigma\).
\end{definition}
We exclude the cases $\l_1 \models \l_2$ and
$\l_2 \models \l_1$ to ensure that the resolvent is not
subsumed by clauses $\alpha$ and $\beta$.
If \(\sigma = \top\), it cannot be represented by
clause since a clause is a disjunction of literals
over distinct variables so it cannot be trivial.

\begin{proposition} \label{prop:resolution}
Closing a (discrete) CNF under resolution and removing subsumed clauses
yields the CNF's prime implicates.
\end{proposition}
 
The following proposition shows that we can incrementally
prune clauses that are not variable-minimal after each resolution step. This is significant computationally and is enabled by the property of local fixation (Definition~\ref{def:fixation}) which is satisfied by the general reasons in Proposition~\ref{prop:closed form} and their CNFs.

\begin{proposition}\label{prop:inc vd clauses}
Let $S$ be a set of clauses (i.e., CNF) that is locally
fixated. For any clauses $\sigma$ and $\sigma'$ in $S$, if $\SetV(\sigma') \subset \SetV(\sigma)$, then the variable-minimal prime implicates of $S$ are the variable-minimal prime implicates of $S \setminus \{\sigma\}$.
\end{proposition} 

\def\GNR{\text{GNR}}
In summary, to compute GNRs, we first
convert the general reason in Proposition~\ref{prop:closed form} into a CNF,
then close the CNF under resolution while removing subsumed clauses and ones that are not variable-minimal after each resolution step.

\section{Conclusion}
\label{sec:conclusion}
We considered the notions of sufficient, necessary and
complete reasons which have been playing a fundamental role
in explainable AI recently. We provided generalizations of
these notions for classifiers with non-binary
features (discrete or discretized). We argued that these generalized
notions have more explanatory power and reveal more information
about the underlying classifier. We further provided results
on the properties and computation of these new notions.

\section*{Acknowledgments}
This work has been partially supported by NSF 
grant ISS-1910317.

%
% ---- Bibliography ----
%
% BibTeX users should specify bibliography style 'splncs04'.
% References will then be sorted and formatted in the correct style.
%
\bibliographystyle{splncs04}
\bibliography{DA_AJ}

\clearpage
\appendix

\section{Proofs}
\label{sec:proofs}

\subsection*{Proposition~\ref{prop:alternative_gr}}

We prove Proposition~\ref{prop:alternative_gr} and~\ref{prop:selection semantics} first, which do not depend on Propositions~\ref{prop:commutative},~\ref{prop:forall-relation},~and~\ref{prop:gr semantics}.

\begin{proof}[of Proposition~\ref{prop:alternative_gr}]
The proof will use the following observations:
\begin{eqnarray*}
\text{(A)} & \w \models x_i \AND \Delta \text{ only if } \w \models \Delta|x_i \\
\text{(B)} & \w \models x_i \AND (\Delta|x_i) \text{ only if } \w \models \Delta \end{eqnarray*}
(A)~is justified as follows:
 $\w \models \Delta$ iff $\Delta | \w = \top$, 
 and $(\Delta | x_i) | \w = \Delta | \w$ since $\w \models x_i$; thus, $(\Delta | x_i) | \w = \top$ and $\w \models \Delta | x_i$.
(B)~is justified as follows:
$\w \models \Delta | x_i$ implies $(\Delta | x_i) | \w = \top$ and $\w \models x_i$ implies $\Delta | \w = (\Delta | x_i) | \w$; hence, $\Delta | \w = \top$ and $\w \models \Delta$.

We next prove both directions of the equivalence while noting that
$\ell_j$ in the proposition statement is equivalent to $\NOT{x_j}$.
\begin{description}
\item $\iuq x_i \cdot \Delta \models (\Delta | x_i) \AND \BAND_{j \not = i}( \NOT{x_j} \OR (\Delta | x_j))$.
Suppose $\w \models \iuq x_i \cdot \Delta$.
Then $\w \models\Delta \AND (\Delta | x_i)$ by Definition~\ref{def:forall2}. 
If $\w \models x_k$ for some \(k\), then 
(1)~$\w \models (\NOT{x_j} \OR \Delta | x_j)$ for all $j \not = k$ since $x_k \models \NOT{x_j}$ and
(2)~$\w \models (\NOT{x_k} \OR \Delta | x_k)$ since $\w \models \Delta | x_k$ which follows from $\w \models \Delta$ and $\w \models x_k$ by~(A).
Hence, $\w \models (\NOT{x_k} \OR \Delta | x_k)$ for all $k$ and,
therefore, $\w \models \BAND_{j \not = i} (\NOT{x_j} \OR (\Delta | x_j))$
and $\w \models (\Delta | x_i) \AND \BAND_{j \not = i}( \NOT{x_j} \OR (\Delta | x_j))$. Hence,
$\iuq x_i \cdot \Delta \models (\Delta | x_i) \AND \BAND_{j \not = i}( \NOT{x_j} \OR (\Delta | x_j))$.

\item $(\Delta | x_i) \AND \BAND_{j \not = i}( \NOT{x_j} \OR (\Delta | x_j)) \models \iuq x_i \cdot \Delta$.
Suppose
$\w \models (\Delta | x_i) \AND \BAND_{j \not = i}( \NOT{x_j} \OR (\Delta | x_j))$. 
If $\w \models x_i$, then $\w \models \Delta$ since 
$\w \models \Delta|x_i$ and given~(B). 
If $\w \not \models x_i$, then $\w \models x_k$ for some \(k \neq i\), and
$\w \models \Delta | x_k$ since $\w \models (\NOT{x_k} \OR (\Delta | x_k))$, which implies $\w \models \Delta$ given~(B). 
Hence, $\w \models \Delta$ in either case and also
$\w \models \Delta \AND \Delta | x_i = \iuq x_i \cdot \Delta.$
Therefore, 
$(\Delta | x_i) \AND \BAND_{j \not = i}( \NOT{x_j} \OR (\Delta | x_j)) \models \Delta \AND \Delta | x_i$. 
\end{description}
\end{proof}

\subsection*{Proposition~\ref{prop:selection semantics}}

\begin{proof}[of Proposition~\ref{prop:selection semantics}]
We first prove the semantics of \(\forall \tau \cdot \Delta\) 
and then \(\iuq \tau \cdot \Delta\) by
induction on the length of simple term \(\tau\).

\vspace{2mm}
\noindent \textit{Semantics of \(\forall \tau \cdot \Delta\).}

\begin{description}
\item Base case: \(\tau = x_i\).

By definition of \(\forall\), 
a world $\w \models \forall x_i \cdot \Delta$
iff 
$\w \models (\Delta | x_i) \AND \BAND_{j \not = i} 
(x_i \OR \Delta | x_j).$ 
If $\w \models x_i$, then $\w \models \forall x_i \cdot \Delta$
iff $\w \models \Delta$ by observations (A) and (B) in the proof of Proposition~\ref{prop:alternative_gr}. 
If $\w \not \models x_i$, then $\w \models \forall x_i \cdot \Delta$  
iff $\w \models \Delta | x_j$ for all $j$. 
Hence, $\w \models \forall x_i \cdot \Delta$ iff
$\w \models \Delta$ and ($\w \not \models x_i$ only if
$\w \models \Delta | x_j$ for all $j$).
The condition ``$\w \models \Delta | x_j$ for all $j$''
is equivalent to 
``$\w' \models \Delta$
for \(\w'\) obtained from \(\w\) by changing its state $x_j$ 
if $x_j \neq x_i$.'' 
If $\w \models x_i$, the previous property holds trivially
as there is no such $\w'$.
Hence, $\w \models \forall x_i \cdot \Delta$ iff
$\w \models \Delta$ and $\w' \models \Delta$
for \(\w'\) obtained from \(\w\) by changing its state $x_j$ to any other state
if $x_j \neq x_i$.
The semantics of $\forall x_i \cdot \Delta$ holds.

\item Inductive step: \(\tau = x_i \AND \tau'\).

Suppose the proposition holds for \(\forall \tau' \cdot \Delta\).
We next show that it holds for \(\forall \tau \cdot \Delta\).
Let $\Gamma = \forall \tau' \cdot \Delta$. By the base case,
$\w \models \forall x_i \cdot \Gamma$ iff
(1)~$\w \models \Gamma$ and
(2)~$\w' \models \Gamma$ for \(\w'\) obtained from \(\w\) by changing its state $x_j$ if $x_j \neq x_i$.
By the induction hypothesis, 
(1)~can be replaced by 
``$\w \models \Delta$ and
$\w' \models \Delta$ for \(\w'\) obtained from \(\w\) by changing
the states of variables set differently in $\tau'$.''
Moreover, (2)~can be replaced by
``$\w' \models \Delta$ and $\w'' \models \Delta$ 
for \(\w'\) obtained from \(\w\) by changing its state $x_j$ if $x_j \neq x_i$ and for
$\w''$ obtained from \(\w'\) by changing the states of variables set 
differently in $\tau'$.''
Replacing (1), (2) as suggested above gives:
$\w \models \forall \tau \cdot \Delta$ iff 
$\w \models \Delta$ and
$\w' \models \Delta$ for \(\w'\) obtained from \(\w\) by changing the states of variables set differently in $\tau$.
The semantics of \(\forall \tau \cdot \Delta\) holds.
\end{description}

\noindent \textit{Semantics of \(\iuq \tau \cdot \Delta\).}

\begin{description}
\item Base case \(\tau = x_i\).

By definition of \(\iuq\),  
$\w \models \iuq x_i \cdot \Delta$ iff 
$ \w \models\Delta \AND (\Delta | x_i)$. 
We next prove: if \(\w \models \Delta\), then
\(\w \models \Delta | x_i\) is equivalent to
``$\w' \models \Delta$ for $\w'$ obtained by setting 
$X$ to $x_i$ in $\w$,'' which proves the semantics
of \(\iuq x_i \cdot \Delta\).
Suppose \(\w \models \Delta\).
We next show both directions of the equivalence.
\begin{description}
\item Suppose \(\w \models \Delta | x_i\) and let \(\w'\) be a
world obtained by setting variable $X$ to $x_i$ in world $\w$.
Then $\w' \models \Delta | x_i$ given 
$\w \models \Delta | x_i$ and since 
$\Delta | x_i$ does not mention variable $X$. 
Hence, $\w' \models \Delta$ by observation (B) in the proof of Proposition~\ref{prop:alternative_gr}.

\item Suppose $\w' \models \Delta$ for $\w'$ obtained by setting 
$X$ to $x_i$ in $\w$.
Then $\w' \models \Delta | x_i$ by observation (A) in the proof of Proposition~\ref{prop:alternative_gr}. 
Moreover, $\w \models \Delta | x_i$ given $\w' \models \Delta | x_i$
and since $\Delta | x_i$ does not mention variable $X$.
\end{description}
This proves the semantics of \(\iuq x_i \cdot \Delta\).

\item Inductive step: \(\tau = x_i \AND \tau'\).

Suppose the proposition holds for \(\iuq \tau' \cdot \Delta\).
We next show that it holds for \(\iuq \tau \cdot \Delta\).
Let $\Gamma = \iuq \tau' \cdot \Delta$. By the base case,
$\w\models\iuq x_i \cdot \Gamma$ iff
(1)~$\w \models \Gamma$ and
(2)~$\w' \models \Gamma$ for \(\w'\) obtained from \(\w\) by setting
\(X\) to  $x_i$. By the induction hypothesis,
(1)~can be replaced by ``\(\w \models \Delta\) and \(\w' \models \Delta\) for \(\w'\) obtained from \(\w\) by setting some variables to their states in \(\tau'\).''
Moreover, (2)~can be replaced by ``\(\w' \models \Delta\) 
and \(\w'' \models \Delta\)
for \(\w'\) obtained from \(\w\) by setting \(X\) to \(x_i\)
and \(\w''\) obtained from \(\w'\) by
setting some variables to their states in \(\tau'\).''
Replacing (1), (2) as suggested above gives:
$\iuq \tau \cdot \Delta$ iff
$\w \models \Delta$ and
$\w' \models \Delta$
for \(\w'\) obtained from \(\w\) by setting some variables to their states in $\tau,$ which proves the semantics of 
\(\iuq \tau \cdot \Delta\). 
\end{description}
\end{proof}

\subsection*{Proposition~\ref{prop:commutative}}

\begin{proof}[of Proposition~\ref{prop:commutative}]
We have:
\begin{align*}
\iuq y \cdot (\iuq  x \cdot \Delta) &= \iuq y \cdot (\Delta \AND \Delta | x)\\
&= (\Delta \AND \Delta | x) \AND ((\Delta \AND \Delta | x) | y)\\
&= (\Delta \AND \Delta | x) \AND (\Delta | y \AND \Delta | x, y)\\
&= (\Delta \AND \Delta | y) \AND (\Delta | x \AND \Delta | y, x)\\
&= (\Delta \AND \Delta | y) \AND ((\Delta \AND \Delta | y) | x)\\
&= (\iuq  y \cdot \Delta) \AND ((\iuq  y \cdot \Delta) | x) \\
&= \iuq x \cdot (\iuq  y \cdot \Delta). 
\end{align*}
\end{proof}

\subsection*{Proposition~\ref{prop:forall-relation}}

\begin{proof}[of Proposition~\ref{prop:forall-relation}]
It suffices to prove that $\forall \instance \cdot \Delta \models \iuq \instance \cdot \Delta \models \Delta$ since $\instance \models \forall \instance \cdot \Delta$~\cite{jair/DarwicheM21}.
By multiple applications of Definition~\ref{def:forall2}, $\iuq \instance \cdot \Delta$ is equivalent to $\Delta \AND \Gamma$ for some formula $\Gamma$. Thus, $\iuq \instance \cdot \Delta \models \Delta$. 
Moreover, $\forall \instance \cdot \Delta \models \iuq \instance \cdot \Delta$
by Proposition~\ref{prop:selection semantics} (already proven). 
\end{proof}

\subsection*{Proposition~\ref{prop:gr semantics}}

\begin{proof}[of Proposition~\ref{prop:gr semantics}]
By Proposition~\ref{prop:alternative_gr} (already proven),
\(\iuq x_i \cdot \Delta\) can be written as an NNF over formulas
that either do not mention variable \(X\) 
or are \(X\)-literals implied by \(x_i\).
Hence, $\iuq \instance \cdot \Delta$
can always be written as an NNF whose literals are implied by
instance \(\instance\) (by repeated application of the previous observation).
We also have $\instance \models \iuq \instance \cdot \Delta \models \Delta$ by Proposition~\ref{prop:forall-relation}. Hence, it suffices to show that if \(\Gamma\) is an NNF such that 
(1)~\(\instance\) satisfies the literals of \(\Gamma\) and
(2)~$\instance \models \Gamma \models \Delta$, then 
$\Gamma \models \iuq \instance \cdot \Delta$.
We next prove this by contradiction.
Suppose $\Gamma$ is an NNF that satisfies properties~(1) and~(2),
and \(\Gamma \not \models \iuq \instance \cdot \Delta\).
Then
$\w \models \Gamma$ and $\w \not \models \iuq \instance \cdot \Delta$ for some world \(\w\).
Let $\w'$ be a world obtained from $\w$ by setting some variables in $\w$ to their states in $\instance$. Then $\w' \models \Gamma$ 
since \(\instance\) satisfies all
literals of \(\Gamma\) by~(1), and \(\instance \models \Gamma\) by~(2).
We now have \(\w \models \Gamma \models \Delta\) by~(2), and 
$\w' \models \Gamma \models \Delta$ for all such worlds \(\w'\), 
which implies
$\w \models \iuq \instance \cdot \Delta$ by
Proposition~\ref{prop:selection semantics} (already proven).
This is a contradiction so
$\Gamma \models \iuq \instance \cdot \Delta$.
\end{proof}

\subsection*{Proposition~\ref{prop:quantify-b}}

\begin{proof}[of Proposition~\ref{prop:quantify-b}]
If $x \models \lit$, then $\iuq x \cdot \lit = \lit \AND \lit | x = \lit \AND \top = \lit$.
If $x \not \models \lit$, then $\iuq x \cdot \lit = \lit \AND \lit | x = \lit \AND \bot = \bot$.
If \(X\) does not appear in $\Delta$, 
then $\iuq x \cdot \Delta = \Delta \AND \Delta | x = \Delta \AND \Delta = \Delta$. 
\end{proof}

\subsection*{Proposition~\ref{prop:distribute-and-or}}

\begin{proof}[of Proposition~\ref{prop:distribute-and-or}]
For the distribution over conjuncts,
\begin{align*}
\iuq x_i \cdot (\alpha \AND \beta) &= (\alpha \AND \beta) \AND (\alpha \AND \beta) | x_i\\
&= \alpha \AND \beta \AND \alpha | x_i \AND \beta | x_i\\
&= (\alpha \AND \alpha | x_i) \AND (\beta \AND \beta | x_i)\\
&= (\iuq x_i \cdot \alpha) \AND (\iuq x_i \cdot \beta).
\end{align*}
For the distribution over disjuncts, suppose variable \(X\) does not
occur in \(\alpha\). Then $\iuq x_i \cdot \alpha = \alpha$ by Proposition~\ref{prop:quantify-b}. Moreover, 
\begin{align*}
\iuq x_i \cdot (\alpha \OR \beta) 
&= (\alpha \OR \beta) \AND (\alpha \OR \beta) | x_i \\
&= (\alpha \OR \beta) \AND (\alpha | x_i \OR \beta | x_i)\\
&= (\alpha \OR \beta) \AND (\alpha  \OR \beta | x_i)\\
&= \alpha + (\alpha \AND (\beta | x_i)) + (\alpha \AND \beta) + (\beta \AND (\beta | x_i)) \\ 
&= \alpha \OR (\beta \AND (\beta | x_i))\\
&= \iuq x_i \cdot \alpha \OR \iuq x_i \cdot \beta.
\end{align*}
The proof is symmetric when $X$ does not occur in $\beta$.
\end{proof}

\subsection*{Proposition~\ref{prop:gsr_no_lost}}

We prove Proposition~\ref{prop:gsr_no_lost} first, which does not depend on Proposition~\ref{prop:redundant sr}.

\begin{lemma}\label{lemma:gsr_no_lost}
Let $\instance$ be an instance, $\tau$ be a simple term and $\tau'$
be a GSR for the decision on $\instance$. Then $\tau = \instance \mycap \tau'$ iff $\instance \models \tau \models \tau'$ and $\SetV(\tau) = \SetV(\tau')$.
\end{lemma}

\begin{proof}
Let $\instance$ be an instance, $\tau$ be a simple term and $\tau'$
be a GSR for the decision on $\instance$.
We next prove both directions of the equivalence.

\begin{description}
\item $\tau = \instance \mycap \tau'$ only if $\instance \models \tau \models \tau'$ and $\SetV(\tau) = \SetV(\tau')$.

Suppose $\tau = \instance \mycap \tau'$. Recall that \(\instance \mycap \tau'\) denotes the smallest subterm in \(\instance\) that implies \(\tau'\). Hence, $\instance \models \instance \mycap \tau' \models \tau'$ and $\instance \models \tau \models \tau'$. Moreover, $\SetV(\instance \mycap \tau') = \SetV(\tau')$ by definition of $\mycap$ so
$\SetV(\tau) = \SetV(\tau')$.

\item $\instance \models \tau \models \tau'$ and $\SetV(\tau) = \SetV(\tau')$ only if  $\tau = \instance \mycap \tau'$.

Suppose $\instance \models \tau \models \tau'$ and $\SetV(\tau) = \SetV(\tau')$. Since $\instance \models \tau \models \tau'$ and term $\tau$ is simple, then (1)~$\tau$ is a subterm in $\instance$ and (2)~~$\tau$ implies $\tau'$. It then suffices to show that no strict subset of $\tau$ satisfies~(1) and~(2). Since $\SetV(\tau) = \SetV(\tau')$, and $\SetV(\instance \mycap \tau') = \SetV(\tau')$ by definition of $\mycap$, we get $\SetV(\tau) = \SetV(\instance \mycap \tau')$. Hence, no strict subset of $\tau$ satisfies~(1) and~(2), so $\tau = \instance \mycap \tau'$.
\end{description}
\end{proof}

\begin{proof}[of Proposition~\ref{prop:gsr_no_lost}]
Let $\instance$ be an instance and $\tau$ be a simple term. Given Lemma~\ref{lemma:gsr_no_lost}, it suffices to show that $\tau$ is a SR iff $\instance \models \tau \models \tau'$ and $\SetV(\tau) = \SetV(\tau')$ for some GSR $\tau'$. Recall that $\tau$ is a SR iff 
(1)~$\instance \models \tau \models \Delta$ and 
(2)~$\tau \models \tau'' \models \Delta$ for simple term $\tau''$ only if $\tau = \tau''$.  We next prove both directions of equivalence.

\begin{description}
\item $\tau$ is a SR only if $\instance \models \tau \models \tau'$ and $\SetV(\tau) = \SetV(\tau')$ for some GSR $\tau'$. 

Suppose $\tau$ is a SR.
Then $\instance \models \tau \models \Delta$ and $\tau \models \tau'' \models \Delta$ for simple term $\tau''$ only if $\tau = \tau''$.  If $\tau$ is a GSR, then $\instance \models \tau \models \tau$ and $\SetV(\tau) = \SetV(\tau)$ so the result holds trivially. 
Suppose $\tau$ is not a GSR.
By definition of a GSR and 
$\instance \models \tau \models \Delta$, there must exist a GSR $\tau'$ such that $\tau \models \tau' \models \Delta$ and $\tau' \not = \tau$. Since $\tau \models \tau'$, $\SetV(\tau') \subseteq \SetV(\tau)$. Moreoever, by Proposition~\ref{prop:redundant sr}, $\instance \mycap \tau' \models \Delta$ and $\instance \mycap \tau'$ is a simple term. Therefore, $\SetV(\instance \mycap \tau') = \SetV(\tau') \subseteq \SetV(\tau)$.
Since $\SetV(\instance \mycap \tau') \subseteq \SetV(\tau)$ and
both $\instance \mycap \tau'$ and $\tau$ are simple terms implied by $\instance$,  we get
$\tau \models \instance \mycap \tau' \models \Delta$. 
Since \(\tau\) is a SR, we now have
$\tau = \instance \mycap \tau'$, 
so $\SetV(\tau') = \SetV(\tau)$. Hence, $\instance \models \tau \models \tau'$ and $\SetV(\tau) = \SetV(\tau')$ for GSR $\tau'$.

\item  $\instance \models \tau \models \tau'$ and $\SetV(\tau) = \SetV(\tau')$ for some GSR $\tau'$ only if $\tau$ is a SR. 

Suppose $\instance \models \tau \models \tau'$ and $\SetV(\tau) = \SetV(\tau')$ for some GSR $\tau'$.
By definition of a GSR, $\tau' \models \Delta$
and, hence, 
(1)~$\instance \models \tau \models \tau' \models \Delta$. 
Suppose now that $\tau \models \tau'' \models \Delta$ and $\tau \not = \tau''$ for some simple term $\tau''$. We will next show a contradiction
which implies (2) $\tau \models \tau'' \models \Delta$ only if
$\tau = \tau''$ for any simple term $\tau''$.
Let $\tau''$ be the weakest simple term satisfying our supposition. 
We then have $\SetV(\tau'') \subset \SetV(\tau)$. 
Moreover, $\tau''$ must be a SR.
By the first direction, there exists a GSR $\tau'''$ where $\SetV(\tau''') = \SetV(\tau'') \subset \SetV(\tau) = \SetV(\tau')$.
Hence \(\tau'\) is not variable-minimal (compared to $\tau'''$) so
it cannot be a GSR, a contradiction.
Hence, (2) holds.
Given (1) and (2), \(\tau\) is a SR.
\end{description}

\end{proof}

\subsection*{Proposition~\ref{prop:redundant sr}}

\begin{proof}[of Proposition~\ref{prop:redundant sr}]
Suppose $\instance$ is an instance in class 
$\Delta$ and $\tau$ is a weakest term s.t.
$\instance \models \tau \models \Delta$. We next
prove both parts of the proposition. 

\textit {Part~$1$.}
Suppose $\tau'$ is a weakest term s.t.
$\instance \models \tau' \models \Delta$ and
$\vars(\tau') \subset \vars(\tau)$.
We will next show $\instance \mycap \tau \models \instance \mycap \tau' \models \Delta$.
Since $\instance \models \tau$ and $\instance \models \tau'$, then $\instance \models \l$
for every literal $\l$ in $\tau$ or $\tau'$. Hence, 
$\instance \mycap \tau$ is the subset \(\JJ\) of \(\instance\) such
that \(\SetV({\JJ}) = \SetV(\tau)\) and
$\instance \mycap \tau'$ is the subset \(\JJ'\) of \(\instance\) such
that \(\SetV(\JJ') = \SetV(\tau')\).
Since $\SetV(\tau') \subset \SetV(\tau)$, 
\(\SetV(\JJ') \subset \SetV(\JJ)\) and, hence,
$\instance  \mycap \tau = \JJ  \models \JJ' = \instance\mycap \tau'$.
Moreover, since $\instance \mycap \tau' \models \tau'$ and $\tau' \models \Delta$, we get
$\instance \mycap \tau \models \instance \mycap \tau' \models \Delta$.

\textit {Part~$2(a)$.}
Suppose $\tau'$ is a weakest term s.t.
$\instance \models \tau' \models \Delta$ and
$\vars(\tau') \subset \vars(\tau)$.
By Part~$1$, $\instance \models \instance \mycap \tau \models \instance \mycap \tau' \models \Delta$. Hence, $\instance \mycap \tau$ is not a SR since
$\instance \mycap \tau'$ is weaker than 
$\instance \mycap \tau$, yet $\instance \models \instance \mycap \tau' \models \Delta$.

\textit {Part~$2(b)$.}
Suppose there is no weakest term $\tau'$ s.t.
$\instance \models \tau' \models \Delta$ and
$\vars(\tau') \subset \vars(\tau)$.
Then $\tau$ is a GSR. By Proposition~\ref{prop:gsr_no_lost}, $\instance \mycap \tau$ is a SR. 
\end{proof}

\subsection*{Proposition~\ref{prop:gnr_no_lost}}

We prove Proposition~\ref{prop:gnr_no_lost} first, which does not depend on Proposition~\ref{prop:redundant nr}.

\begin{lemma}\label{lemma:gnr_no_lost}
Let $\instance$ be an instance, $\sigma$ be a simple clause, and $\sigma'$ be a GNR for the decision on $\instance$. Then $\sigma = \instance \mycap \sigma'$ iff $\instance \mymodels \sigma \models \sigma'$ and $\SetV(\sigma) = \SetV(\sigma')$.
\end{lemma}

\begin{proof}
Suppose $\instance$ is an instance, $\sigma$ is a simple clause, and $\sigma'$ is a GNR for the decision on $\instance$.
We next prove both directions of the equivalence. 

\begin{description}
\item $\sigma = \instance \mycap \sigma'$ only if $\instance \mymodels \sigma \models \sigma'$ and $\SetV(\sigma) = \SetV(\sigma')$.

Suppose $\sigma = \instance \mycap \sigma'$. Recall that $\instance \mycap \sigma'$ denotes the disjunction of states that appear in both $\instance$ and $\sigma'$. Therefore, we have $\instance \mymodels \instance \mycap \sigma'$ and $\instance \mycap \sigma' \models \sigma'$, which implies $\instance \mymodels \sigma \models \sigma'$. Since $\sigma'$ is a GNR, we have $\instance \mymodels \sigma'$.
Therefore, $\SetV(\instance \mycap \sigma') = \SetV(\sigma')$, so $\SetV(\sigma) = \SetV(\sigma')$.

\item $\instance \mymodels \sigma \models \sigma'$ and $\SetV(\sigma) = \SetV(\sigma')$ only if $\sigma = \instance \mycap \sigma'$.

Suppose $\instance \mymodels \sigma \models \sigma'$ and $\SetV(\sigma) = \SetV(\sigma')$. Since $\sigma$ is a simple clause and $\instance \mymodels \sigma$, $\sigma$ is a disjunction of some states $S$ in $\instance$. By definition of $\mycap$, $\instance \mycap \sigma'$ is a disjunction of some states $S'$ in $\instance$. Since $\sigma \models \sigma'$, 
$S \subseteq S'$. 
Since $\SetV(\sigma) = \SetV(\sigma')$,
$S = S'$. Hence, $\sigma = \instance \mycap \sigma'$.
\end{description}
\end{proof}

\begin{proof}[of Proposition~\ref{prop:gnr_no_lost}]
Let instance $\instance$ be in class $\Delta$ and \(\sigma\) be a simple clause. 
By Lemma~\ref{lemma:gnr_no_lost}, it suffices to show that $\sigma$ is a NR iff $\instance \mymodels \sigma \models \sigma^\star$ and $\SetV(\sigma) = \SetV(\sigma^\star)$ for some GNR $\sigma^\star$.
Recall that 
$\sigma$ is a NR for the decision on \(\instance\) 
iff 
(1)~$\instance \mymodels \sigma$ and $(\instance \myminus \sigma) \AND \NOT{\sigma} \not \models \Delta$ and
(2)~a simple clause $\sigma'$ satisfies the previous condition and $\sigma' \models \sigma$ only if $\sigma = \sigma'$. 
We will reference (1) and (2) next as
we prove both directions of the equivalence in Proposition~\ref{prop:gnr_no_lost}.

\begin{description}
\item $\sigma$ is a NR only if $\instance \mymodels \sigma \models \sigma^\star$ and $\SetV(\sigma) = \SetV(\sigma^\star)$ for some GNR $\sigma^\star$. 

Suppose $\sigma$ is a NR. We prove this direction by finding a GNR $\sigma^\star$ that satisfies the properties above. Given (1) and (2),
there is no simple clause $\sigma'$ s.t. $\instance \mymodels \sigma'$, $(\instance \myminus \sigma') \AND \NOT{\sigma'} \not \models \Delta$, 
$\sigma' \models \sigma$ and $\sigma' \not = \sigma$
(equivalent to $\SetV(\sigma') \subset \SetV(\sigma)$).
Hence, there is no GNR $\sigma''$ such
that $\SetV(\sigma'') \subset \SetV(\sigma)$. 
Next, since $(\instance \myminus \sigma) \AND \NOT{\sigma} \not \models \Delta$, there is a world $\w$ such that $\w \models (\instance \myminus \sigma) \AND \NOT{\sigma}$ and $\w \not \models \Delta$. Our goal is to construct a clause $\sigma'$ such that the only world that satisfies $(\instance \myminus \sigma') \AND \NOT{\sigma'}$ is $\w$. This gives us $(\instance \myminus \sigma') \AND \NOT{\sigma'} \models \NOT{\Delta}$. Then, either $\sigma'$ is a GNR, or $\sigma'$ is subsumed by some GNR. If we can find such a clause $\sigma'$, we can also find the sought GNR $\sigma^\star$ that finishes this direction of the proof. This is shown next.
Consider the clause $\sigma'$ equivalent to $\NOT{\w \setminus (\w \myminus \sigma)}$. Note that $\w \models (\instance \myminus \sigma)$ and $\w \models \NOT{\sigma}$ where $\instance \mymodels \sigma$. It follows that $\sigma'$ is equivalent to the negation of a conjunction of literals in $\w$ whose variables are mentioned by $\sigma$. Every $X$-literal $\l'$ in $\sigma'$ is entailed by an $X$-literal $\l$ in $\sigma$ because $\l'$ contains all states of $X$ but the one from $\w$ and $\w \models \NOT{\sigma}$. Thus, $\instance \models \sigma'$. Then $(\instance \myminus \sigma') \AND \NOT{\sigma'} = (\instance \myminus \sigma) \AND \NOT{\sigma'}$, and the only model of $(\instance \myminus \sigma) \AND \NOT{\sigma'}$ is $\w$. Therefore, $(\instance \myminus \sigma') \AND \NOT{\sigma'} \models \NOT{\Delta}$. Note that $\SetV(\sigma') = \SetV(\sigma)$, and we already showed that there is no GNR  $\sigma''$ such that $\SetV(\sigma'') \subset \SetV(\sigma)$. Thus, either $\sigma'$ is a GNR, in which case we let $\sigma^\star = \sigma'$, or $\instance \mymodels \sigma^\star \models \sigma'$ and $\SetV(\sigma^\star) = \SetV(\sigma') = \SetV(\sigma)$ for some GNR $\sigma^\star$. Either way, $\sigma \models \sigma^\star$ follows from $\SetV(\sigma) = \SetV(\sigma^\star)$, $\instance \mymodels \sigma$ and $\instance \mymodels \sigma^\star$.

\item $\instance \mymodels \sigma \models \sigma'$ and $\SetV(\sigma) = \SetV(\sigma')$ for some GNR $\sigma'$ only if $\sigma$ is a NR.

Suppose $\instance \mymodels \sigma \models \sigma'$ and $\SetV(\sigma) = \SetV(\sigma')$ for some GNR $\sigma'$. We next prove~(1) and then prove~(2). 
Since $\sigma'$ is a GNR, $(\instance \myminus \sigma') \AND \NOT{\sigma'} \models \NOT{\Delta}$. 
Moreover, $\instance \myminus \sigma = \instance \myminus \sigma'$ since
$\SetV(\sigma) = \SetV(\sigma')$, 
$\instance \mymodels \sigma$ and $\instance \mymodels \sigma'$.
We also have $\NOT{\sigma'} \models \NOT{\sigma}$,
given $\sigma \models \sigma'$,
which implies $(\instance \myminus \sigma') \AND \NOT{\sigma'} \models (\instance \myminus \sigma) \AND \NOT{\sigma}$. Therefore, 
$\w \models (\instance \myminus \sigma') \AND \NOT{\sigma'} \models \NOT{\Delta}$ only if
$\w \models (\instance \myminus \sigma) \AND \NOT{\sigma}$ and $\w \models \NOT{\Delta}$, which implies $(\instance \myminus \sigma) \AND \NOT{\sigma} \not \models \Delta$. 
Together with $\instance \mymodels \sigma$, this
gives~(1).
We prove (2) by contradiction. Suppose (2) does not hold. Then there exists a NR $\sigma''$ such that $\sigma'' \models \sigma$ and $\sigma'' \not = \sigma$, so $\SetV(\sigma'') \subset \SetV(\sigma)$. By the first direction, there exists a GNR $\sigma'''$ such that $\SetV(\sigma''') = \SetV(\sigma'') \subset \SetV(\sigma) = \SetV(\sigma')$. This is a contradiction since $\sigma'$ is a GNR. Thus, (2)~holds. 
\end{description}
\end{proof}

\subsection*{Proposition~\ref{prop:redundant nr}}

\begin{proof}[of Proposition~\ref{prop:redundant nr}]
Let $\instance$ be an instance in class $\Delta$
and \(\sigma\) be a strongest clause s.t. 
$\instance \mymodels \sigma$ and $(\instance \myminus \sigma) \AND \NOT \sigma \models \NOT{\Delta}$. We next prove both parts of the proposition.

\textit{Part~$1$.}
Suppose $\sigma'$ is a strongest clause s.t.
$\instance \mymodels \sigma'$ and $(\instance \myminus \sigma') \AND \NOT {\sigma'} \models \NOT{\Delta}$
and $\SetV(\sigma') \subset \SetV(\sigma)$.
We next show that $\instance \myminus \sigma' \models \instance \myminus \sigma$.
Since $\instance \mymodels \sigma$ and $\instance \mymodels \sigma'$, every literal $\l$ in $\sigma$ or $\sigma'$ satisfies $\instance \models \l$. Therefore, 
$\instance \myminus \sigma$ is the subset \(\JJ\) of \(\instance\) such
that \(\SetV({\JJ}) = \SetV(\instance) \setminus \SetV(\sigma)\) and
$\instance \mycap \sigma'$ is the subset \(\JJ'\) of \(\instance\) such
that \(\SetV(\JJ') = \SetV(\instance) \setminus \SetV(\sigma')\).
Since $\SetV(\sigma') \subset \SetV(\sigma)$, we have \(\SetV(\JJ) \subset \SetV({\JJ'})\) and, hence, $\instance \myminus \sigma' = \JJ' \models \JJ = \instance \myminus \sigma$. 

\textit{Part~$2(a)$.}
Suppose there is no strongest clause $\sigma'$ s.t.
$\instance \mymodels \sigma'$ and $(\instance \myminus \sigma') \AND \NOT {\sigma'} \models \NOT{\Delta}$
and $\SetV(\sigma') \subset \SetV(\sigma)$.
Then $\sigma$ is a GNR. By Proposition~\ref{prop:gnr_no_lost}, $\instance \mycap \sigma$ is a NR.

\textit{Part~$2(b)$.}
Suppose $\sigma'$ is a strongest clause s.t.
$\instance \mymodels \sigma'$ and $(\instance \myminus \sigma') \AND \NOT {\sigma'} \models \NOT{\Delta}$
and $\SetV(\sigma') \subset \SetV(\sigma)$.
Since $\instance \mymodels \sigma$, $\instance \mymodels \sigma'$, and $\SetV(\sigma') \subset \SetV(\sigma)$, we have $\instance \mycap \sigma' \models \instance \mycap \sigma$ and $\instance \myminus \sigma' = \instance \myminus (\instance \mycap \sigma')$. Since $\NOT{\sigma'} \models \NOT{\instance \mycap \sigma'}$ and $(\instance \myminus \sigma') \AND \NOT{\sigma'} \models \NOT{\Delta}$, we have $(\instance \myminus (\instance \mycap \sigma')) \AND \NOT{\instance \mycap \sigma'} \not \models \Delta$. Hence, $\instance \mycap \sigma$ is not a NR because $\instance \mycap \sigma'$ is stronger than $\instance \mycap \sigma$, yet $\instance \mymodels \instance \mycap \sigma'$, and $(\instance \myminus (\instance \mycap \sigma')) \AND \NOT{\instance \mycap \sigma'} \not \models \Delta$.
\end{proof}

\subsection*{Proposition~\ref{prop:PI-IP}}

\begin{lemma}\label{lemma:formula equal IPs}
Let $\Delta$ be a formula with discrete variables and $\sigma_1, \ldots, \sigma_n$ be the prime implicates of $\Delta$. Then $\Delta$ is equivalent to $\BAND_{i = 1}^n \sigma_i$. That is, $\Delta$ is equivalent to the conjunction of its prime implicates.
\end{lemma}

\begin{proof}
We prove both directions of the equivalence. $\Delta \models \BAND_{i = 1}^n \sigma_i$ since $\Delta \models \sigma_i$ for all $i$. We next prove that $\BAND_{i = 1}^n \sigma_i \models \Delta$ by contradiction. Let $\w$ be a world s.t. $\w \models \BAND_{i=1}^n\sigma_i$ 
but $\w \not \models \Delta$.
Then $\w \models \NOT{\Delta}$ and, hence, 
$\Delta \models \NOT{\w}.$ Since $\NOT{\w}$ is a clause,
it must be subsumed by some prime implicate $\sigma_j.$
Hence, \(\w \models \sigma_j \models \NOT{\w}\) which is a
contradiction.
\end{proof}

\begin{proof}[of Proposition~\ref{prop:PI-IP}]

We first prove the part about GSRs, then the one for GNRs.

{\bf GSRs are the variable-minimal prime implicants of $\iuq \instance \cdot \Delta$.}
We prove two directions next.

\begin{description}
\item All variable-minimal prime implicants of $\iuq \instance \cdot \Delta$ are GSRs. 

Let $\tau$ be a variable-minimal prime implicant of $\iuq \instance \cdot \Delta$. By Proposition~\ref{prop:gr semantics}, it suffices to prove {\bf (1)}~$\tau$ is a weakest term such that $\instance \models \tau \models \Delta$ and {\bf (2)}~no term $\tau'$ satisfies the previous condition if $\SetV(\tau') \subset \SetV(\tau)$. We prove these next.

{\bf (1)}~We have 
$\tau \models \iuq \instance \cdot \Delta \models \Delta$
by Proposition~\ref{prop:forall-relation} and given that
$\tau$ is a prime implicant of $\iuq \instance \cdot \Delta$.
We next prove $\instance \models \tau$ by contradiction. 

Suppose $\instance \not \models \tau$. 
Then $\instance \not \models \l$ for some 
\(X\)-literal $\l$ in $\tau$. 
Let $\l' = \l \cup \{\instance[X]\}$ where
$\instance[X]$ is the state of variable $X$ in $\instance$.
Consider the term $\tau'$ obtained from $\tau$ by replacing literal $\l$ by $\l'$. Then the models of $\tau'$ are the models of $\tau$ plus the worlds $\w'$ obtained from a model $\w$ of $\tau$ by setting the value of variable $X$ to $\instance[X]$. If we can prove $\w' \models \iuq \instance \cdot \Delta$, then it follows that $\tau'$ is an implicant of $\iuq \instance \cdot \Delta$, which gives us a contradiction since $\tau$ is a prime implicant and $\tau \models \tau'$. 
We have $\w' \models \Delta$
by Proposition~\ref{prop:selection semantics} and since $\w \models \tau \models \iuq \instance \cdot \Delta$. Moreover, for any world $\w''$ obtained from $\w'$ by setting some variables to their states in $\instance$, $\w''$ can also be obtained from some model of $\tau$ by setting some variables to their states in $\instance$. Hence, by Proposition~\ref{prop:selection semantics}, $\w'' \models \Delta$. Proposition~\ref{prop:selection semantics} further tells us that $\w' \models \iuq \instance \cdot \Delta$. Thus, $\tau \models \tau' \models \iuq \instance \cdot \Delta$ and $\tau \not = \tau'$, which is a contradiction since $\tau'$ is a prime implicant. Hence, $\instance \models \tau$.

We now have $\instance \models \tau \models \Delta$.
To prove (1), we need to prove that $\tau$ is the weakest term satisfying the previous property. We prove this by contradiction. Suppose $\instance \models \tau \models \tau' \models \Delta$ and $\tau' \not = \tau$ for some term $\tau'$. Let $\w$ be a world such that $\w \models \tau' \models \Delta$. Since $\instance \models \tau'$, all literals in $\tau'$ are consistent with $\instance$. Therefore, 
$\w' \models \tau' \models \Delta$
for any world $\w'$ obtained from $\w$ by setting some variables in $\w$ to their states in $\instance$. By Proposition~\ref{prop:selection semantics}, $\w \models \iuq \instance \cdot \Delta$, which means all models of $\tau'$ are models of $\iuq \instance \cdot \Delta$, so $\tau'$ is an implicant of $\iuq \instance \cdot \Delta$. This is a contradiction since $\tau$ is a prime implicant of $\iuq \instance \cdot \Delta$. Hence, $\tau$ must be a weakest term satisfying 
$\instance \models \tau \models \Delta$, so (1) holds.

When proving (1), we did not use the variable-minimality of $\tau$. Hence, we make the following observation which
we use later in the proof: {\bf (A)} every prime implicant $\tau$ of $\iuq \instance \cdot \Delta$ satisfies $\instance \models \tau \models \Delta$.

{\bf (2)}~We prove this by contradiction. Suppose there exists a weakest term $\tau'$ satisfying $\instance \models \tau' \models \Delta$ and $\SetV(\tau') \subset \SetV(\tau)$. Since $\instance \models \tau'$, all literals in $\tau'$ are consistent with $\instance$. Therefore, for a world $\w \models \tau' \models \Delta$, every $\w'$ obtained from $\w$ by setting some variables in $\w$ to their state in $\instance$ satisfies $\w' \models \tau' \models \Delta$. By Proposition~\ref{prop:selection semantics}, $\tau'$ is an implicant of $\iuq \instance \cdot \Delta$. Since $\SetV(\tau') \subset \SetV(\tau)$, there is a prime implicant $\tau''$ of $\iuq \instance \cdot \Delta$ satisfying $\SetV(\tau'') \subseteq \SetV(\tau') \subset \SetV(\tau)$. This is a contradiction since $\tau$ is variable-minimal. Hence, (2) holds.

\item All GSRs are variable-minimal prime implicants of $\iuq \instance \cdot \Delta$.

Let $\tau$ be a GSR. We will prove {\bf (1)}~$\tau$ is a prime implicant of $\iuq \instance \cdot \Delta$ and {\bf (2)}~there is no prime implicant $\tau'$ of $\iuq \instance \cdot \Delta$ satisfying $\SetV(\tau') \subset \SetV(\tau)$.

{\bf (1)}~Since $\tau$ is a GSR, $\tau$ is a weakest term such that $\instance \models \tau \models \Delta$. Therefore, all literals in $\tau$ are consistent with $\instance$. Thus, for a world $\w \models \tau \models \Delta$, every world $\w'$ obtained from $\w$ by setting some variables in $\w$ to their states in $\instance$ satisfies $\w' \models \tau \models \Delta$. By Proposition~\ref{prop:selection semantics}, $\w \models \iuq \instance \cdot \Delta$, so $\tau$ is an implicant of $\iuq \instance \cdot \Delta$. 
To prove (1), we next show that $\tau$ is prime by contradiction. Suppose $\tau$ is not prime. 
Then there must be a prime implicant $\tau'$ of $\iuq \instance \cdot \Delta$ satisfying $\tau \models \tau' \models \iuq \instance \cdot \Delta$ and $\tau \not = \tau'$. By observation (A) in the first direction, $\tau'$ satisfies $\instance \models \tau' \models \Delta$. This is a contradiction since $\tau$ is a GSR but not the weakest given $\instance \models \tau \models \tau' \models \Delta$. Hence, $\tau$ is prime, so (1) holds.

{\bf (2)}~We prove this by contradiction. 
Assume (2) does not hold. Then
there are some prime implicants $\tau'$ of $\iuq \instance \cdot \Delta$ satisfying $\SetV(\tau') \subset \SetV(\tau)$. 
Let $\tau''$ be a variable-minimal prime implicant among all such prime implicants $\tau'$. 
By the first direction, $\tau''$ is a GSR, which implies $\tau$ is not a GSR because $\tau$ is not variable minimal. This is a contradiction, so (2) holds.

\end{description}

{\bf GNRs are the variable-minimal prime implicates of $\iuq \instance \cdot \Delta$.}
We prove both directions next.

\begin{description}
\item All variable-minimal prime implicates of $\iuq \instance \cdot \Delta$ are GNRs. 

Let $\sigma$ be a variable-minimal prime implicate of $\iuq \instance \cdot \Delta$. Our goal is to prove {\bf (1)} $\sigma$ is a strongest clause satisfying $\instance \mymodels \sigma$ and $(\instance \myminus \sigma) \AND \NOT{\sigma} \models \NOT{\Delta}$ and {\bf (2)} no clause $\sigma'$ satisfies the previous condition if $\SetV(\sigma') \subset \SetV(\sigma)$.

{\bf (1)}~We first prove $\instance \mymodels \sigma$. Assume the opposite: there is an \(X\)-literal $\l$ in $\sigma$ such that $\instance \not \models \l$. Let $\sigma'$ be a clause obtained from $\sigma$ by removing $\l$. If we show $\iuq \instance \cdot \Delta \models \sigma'$, we get a contradiction because $\sigma$ is a prime implicate of $\iuq \instance \cdot \Delta$ and $\sigma' \models \sigma$. This would prove
$\instance \mymodels \sigma$. 

We show $\iuq \instance \cdot \Delta \models \sigma'$ by showing $\iuq \instance \cdot \Delta \not \models \sigma'$ is impossible. $\iuq \instance \cdot \Delta \not \models \sigma'$ only if $\w \models \iuq \instance \cdot \Delta \models \sigma$  
and $\w \not \models \sigma'$ for some world \(\w\) 
(recall, $\sigma$ is a prime implicate of $\iuq \instance \cdot \Delta$). 
Such a world $\w$ satisfies $\w \models \l$ (since $\l$ is the only distinction from $\sigma$ and $\sigma'$) and, hence, $\w \models \NOT{\sigma'}$. If such a world $\w$ exists, let $\w'$ be a world obtained from 
such a $\w$ by setting variable \(X\) to its state in $\instance$. Note that $\sigma'$ does not mention variable $X$, so $\w' \models \NOT{\sigma'}$. 
We have $\w' \models \Delta$
by Proposition~\ref{prop:selection semantics} since $\w \models \iuq \instance \cdot \Delta$. Moreover, for any world $\w''$ obtained from $\w'$ by setting some variables to their states in $\instance$, $\w''$ can also be obtained from $\w$ by setting some variables to their state in $\instance$. Hence, by Proposition~\ref{prop:selection semantics}, $\w'' \models \Delta$ for all such $\w''$, which means $\w' \models \iuq \instance \cdot \Delta$. 
Finally, $\w' \not \models \sigma$
since $\w' \models \NOT{\sigma'}$ and $\w' \not \models \l$. This is a contradiction since $\iuq \instance \cdot \Delta \models \sigma$. Hence, no such world \(\w\) exists, which shows $\iuq \instance \cdot \Delta \not \models \sigma'$ is impossible. Thus, $\iuq \instance \cdot \Delta \models \sigma'$, which gives us another contradiction. Hence, $\instance \mymodels \sigma$. 

We did not use the fact that $\sigma$ is variable-minimal when proving $\instance \mymodels \sigma$. Therefore, we have the following observation which we use later in the proof: 
{\bf (B)} $\instance \mymodels \sigma$ holds for any prime implicate $\sigma$ of $\iuq \instance \cdot \Delta$.

We next prove $(\instance \myminus \sigma) \AND \NOT{\sigma} \models \NOT{\Delta}$. Since $\instance \mymodels \sigma$, $\SetV(\instance \myminus \sigma) = \SetV(\instance) \setminus \SetV(\sigma)$. Let $\w$ be a world such that $\w \models (\instance \myminus \sigma) \AND \NOT{\sigma}$. We next prove $\w \models \NOT{\Delta}$ by contradiction. Assume $\w \models \Delta$. Then for every world $\w'$ obtained from $\w$ by setting some variables to their states in $\instance$, if $\w' \not = \w$, then $\w' \models \sigma$ because $\instance \mymodels \sigma$ and $\w \models (\instance \myminus \sigma)$. By observation (B) above, all literals in every prime implicate of $\iuq \instance \cdot \Delta$ are consistent with $\instance$. Then, consider any $\w' \not = \w$, which immediately implies $\w' \models \sigma$. $\w'$ must satisfy all prime implicate of $\iuq \instance \cdot \Delta$. Otherwise, since the subset of $\w'$ that disagrees with $\instance$ mentions fewer variables than $\sigma$, $\sigma$ cannot be a variable-minimal prime implicate. By Lemma~\ref{lemma:formula equal IPs}, since $\w'$ satisfies all prime implicate of $\iuq \instance \cdot \Delta$, $\w' \models \iuq \instance \cdot \Delta$ for all such $\w' \not = \w$. Since $\w' \models \Delta$ and $\w \models \Delta$, by Proposition~\ref{prop:selection semantics}, $\w \models \iuq \instance \cdot \Delta$, which means $\w \models \sigma$. This is a contradiction, since $\w \models \NOT{\sigma}$. Hence, $\w \models \NOT{\Delta}$, which implies $(\instance \myminus \sigma) \AND \NOT{\sigma} \models \NOT{\Delta}$.

We now prove that $\sigma$ is a strongest clause satisfying $\instance \mymodels \sigma$ and $(\instance \myminus \sigma) \AND \NOT{\sigma} \models \NOT{\Delta}$,
which finishes the proof of (1). We prove this by contradiction. Assume there is a clause $\sigma'$ such that $\instance \mymodels \sigma'$, $(\instance \myminus \sigma') \AND \NOT{\sigma'} \models \NOT{\Delta}$, $\sigma' \models \sigma$, and $\sigma' \not = \sigma$. If we prove $\sigma'$ is an implicate of $\iuq \instance \cdot \Delta$, we get a contradiction. Consider $\w \models \iuq \instance \cdot \Delta$. If we prove $\w \models \sigma'$, then $\iuq \instance \cdot \Delta \models \sigma'$ follows. Assume $\w \not \models \sigma'$, i.e., $\w \models \NOT{\sigma'}$. Let $\w'$ be a world obtained from $\w$ by setting all variables mentioned by $\w$ but not by $\sigma'$ in $\w$ to their states in $\instance$. Then, $\w' \models \NOT{\sigma'}$ since the variables mentioned by $\sigma'$ are unchanged. Moreover, $\w' \models (\instance \myminus \sigma')$ since $\instance \mymodels \sigma'$, so $\w' \models (\instance \myminus \sigma') \AND \NOT{\sigma'} \models \NOT{\Delta}$. However, by Proposition~\ref{prop:selection semantics}, $\w' \models \Delta$ because $\w \models \iuq \instance \cdot \Delta$. This is a contradiction, so $\w \models \sigma'$, which shows $\sigma'$ is an implicate. 
Hence, (1) holds.

In the previous paragraph, we proved a property which we use later in the proof:
{\bf (C)} Every clause $\sigma'$ satisfying $\instance \mymodels \sigma'$ and $(\instance \myminus \sigma') \AND \NOT{\sigma'} \models \NOT{\Delta}$ is an implicate of $\iuq \instance \cdot \Delta$.

{\bf (2)}~We prove this by contradiction. Assume there is a strongest clause $\sigma'$ that satisfies $\instance \mymodels \sigma'$, $(\instance \myminus \sigma') \AND \NOT{\sigma'} \models \NOT{\Delta}$, and $\SetV(\sigma') \subset \SetV(\sigma)$. By property (C) above, $\sigma'$ is an implicate of $\iuq \instance \cdot \Delta$. Since $\SetV(\sigma') \subset \SetV(\sigma)$, $\sigma$ cannot be variable-minimal, which is a contradiction. Thus, (2) holds.

\item All GNRs are variable-minimal prime implicates of $\iuq \instance \cdot \Delta$. 

Let $\sigma$ be a GNR. We next
prove {\bf (1)} $\sigma$ is a strongest clause satisfying $\iuq \instance \cdot \Delta \models \sigma$ and {\bf (2)} there is no clause $\sigma'$ such that $\SetV(\sigma') \subset \SetV(\sigma)$ and $\sigma'$ satisfies the previous condition, i.e., $\sigma'$ is a prime implicate of $\iuq \instance \cdot \Delta$.

{\bf (1)}~We first prove that $\iuq \instance \cdot \Delta \models \sigma$, then prove $\sigma$ is the strongest such clause. Consider $\w \models \iuq \instance \cdot \Delta$. If we prove $\w \models \sigma$, then $\iuq \instance \cdot \Delta \models \sigma$ follows. Assume $\w \not \models \sigma$, i.e., $\w \models \NOT{\sigma}$. Let $\w'$ be a world obtained from $\w$ by setting all variables mentioned by $\w$ but not by $ \sigma$ in $\w$ to their states in $\instance$. Then, $\w' \models \NOT{\sigma}$ since the variables mentioned by $\sigma$ are unchanged. Moreover, $\w' \models (\instance \myminus \sigma) $ since $\instance \mymodels \sigma$, so $\w' \models (\instance \myminus \sigma) \AND \NOT{\sigma} \models \NOT{\Delta}$ 
since $\sigma$ is a GNR. However, by Proposition~\ref{prop:selection semantics}, $\w' \models \Delta$ because $\w \models \iuq \instance \cdot \Delta$. This is a contradiction, so $\w \models \sigma$, which implies $\iuq \instance \cdot \Delta \models \sigma$.

We next prove $\sigma$ is the strongest clause satisfying $\iuq \instance \cdot \Delta \models \sigma$, which finishes the proof of (1). Assume $\sigma$ is not the strongest, which means there is a clause $\sigma'$ satisfying $\iuq \instance \cdot \Delta \models \sigma' \models \sigma$. Let $\sigma'$ be the weakest such clause, i.e., $\sigma'$ is a prime implicate of $\iuq \instance \cdot \Delta$. If $\sigma'$ is not variable-minimal, or if $\SetV(\sigma') \subset \SetV(\sigma)$, then there must exist a variable-minimal prime implicate of $\iuq \instance \cdot \Delta$ that mentions no more variables than $\sigma'$. By the first direction, this variable-minimal prime implicate is a GNR, which means there is a GNR that mentions fewer variables than $\sigma$. Therefore, $\sigma$ cannot be a GNR, which is a contradiction. Hence, $\sigma$ must be variable-minimal and $\SetV(\sigma') \supseteq \SetV(\sigma)$. Since $\sigma' \models \sigma$, we have $\SetV(\sigma') = \SetV(\sigma)$. By the first direction, since $\sigma'$ is a variable-minimal prime implicate of $\iuq \instance \cdot \Delta$, $\sigma'$ is a GNR. This is contradiction since $\sigma$ is a GNR and $\sigma' \models \sigma$, 
so (1) holds.

{\bf (2)}~We prove this by contradiction. Assume there are some clauses $\sigma'$ that are prime implicates of $\iuq \instance \cdot \Delta$ and $\SetV(\sigma') \subset \SetV(\sigma)$. Let $\sigma''$ be a variable-minimal prime implicate among those prime implicates. By the first direction, $\sigma''$ is a GNR that mentions fewer variables than $\sigma$. This is a contradiction since $\sigma$ is a GNR. Hence, (2) holds.
\end{description}
\end{proof}

\subsection*{Proposition~\ref{prop:closed form}}

The proof of Proposition~\ref{prop:closed form} uses the next
two lemmas which we state and prove first.

\begin{lemma}\label{lemma:iuq or distribute}
Let $\alpha$ be an NNF and \(\l\) by an \(X\)-literal.
If \(\l \models \l'\) 
for every \(X\)-literal \(\l'\) that occurs in \(\alpha\)
then
\[
\iuq x_i \cdot (\alpha \OR \l) = 
\iuq x_i \cdot \alpha \OR \iuq x_i \cdot \l.
\]
\end{lemma}

\begin{proof}
We consider two cases. 
\begin{description}
\item Case $x_i \models \l$.

Then $\iuq x_i \cdot (\alpha \OR \l) 
= (\alpha \OR \l) \AND (\alpha | x_i \OR \l | x_i) 
= (\alpha \OR \l) \AND (\alpha | x_i \OR \top) 
= \alpha \OR \l 
= \alpha \OR \iuq x_i \cdot \l$. 
We next show $\iuq x_i \cdot \alpha = \alpha$.
If \(\l \models \l'\) then \(x_i \models \l'\).
Hence, $\alpha \models \alpha | x_i$ since \(\alpha | x_i\)
is obtained by replacing every \(X\)-literal
\(\l'\) in \(\alpha\) with \(\top\).
We now have $\alpha \models \alpha \AND \alpha | x_i$
and, hence, $\alpha=\alpha \AND \alpha | x_i=\iuq x_i \cdot \alpha$.  Thus, $\iuq x_i \cdot (\alpha \OR \l) = \iuq x_i \cdot \alpha \OR \iuq x_i \cdot \l$.

\item Case $x_i \not \models \l$.

Since $\iuq x_i \cdot \l = \l \AND \l | x_i = \l \AND \bot = \bot$,
it suffices to show $\iuq x_i \cdot (\alpha \OR \l) = \iuq x_i \cdot \alpha$. We have
$\iuq x_i \cdot (\alpha \OR \l) 
= (\alpha \OR \l) \AND (\alpha | x_i \OR \l | x_i)
= (\alpha \OR \l) \AND (\alpha | x_i \OR \bot) 
= (\alpha \OR \l) \AND \alpha | x_i 
= (\alpha \AND \alpha | x_i) \OR (\l \AND \alpha | x_i) 
= \iuq x_i \cdot \alpha \OR (\l \AND \alpha | x_i)$. 
We next show that $\l \AND \alpha | x_i \models \iuq x_i \cdot \alpha$
which finishes the proof.
We have $\alpha | x_i=\alpha | x_j$ for all $x_j \models \l$
since \(\l \models \l'\) for every \(X\)-literal \(\l'\) in $\alpha$. 
Hence,
$\l \AND \alpha | x_i 
= \BOR_{x_j \models \l} (x_j \AND \alpha | x_i) 
= \BOR_{x_j \models \l} (x_j \AND \alpha | x_j)
= \BOR_{x_j \models \l} (x_j \AND \alpha)
= \l \AND \alpha$ which implies
$\l \AND \alpha | x_i \models \alpha \AND \alpha | x_i 
= \iuq x_i \cdot \alpha$.
\end{description}
\end{proof}

\begin{lemma}\label{lemma:decision graph NNF}
The formula of class \(c\) in decision graph $T$ is equivalent to
an NNF $\Delta^c[T]$ defined as follows:
\[
\Delta^c[T] = 
\begin{cases}
    \top & \text{ if $T$ has class $c$}\\
    \bot & \text{ if $T$ has a class $c' \not = c$}\\
    \BAND_j(\Delta^c[T_j] \OR \l) & \text{ if $T$ has edges $ \xrightarrow{X, S_j} T_j$}
\end{cases}
\]
where $\l$ is the \(X\)-literal $\{x_i | x_i \not \in S_j\}$.
\end{lemma}

\begin{proof}
This result is proven in~\cite{DBLP:conf/aaai/DarwicheJ22}.
\end{proof}

\begin{proof}[of Proposition~\ref{prop:closed form}]
We will show how to compute the general reason $\iuq \instance \cdot \Delta^c[T]$ by using the definition of class formula $\Delta^c[T]$
as given by Lemma~\ref{lemma:decision graph NNF}.
By Proposition~\ref{prop:distribute-and-or}, $\iuq$ distributes over the and-nodes of $\Delta^c[T]$. 
Every disjunction in this NNF has the form $\Delta^c[T_j] \OR \l$ where $\l = \{x_i | x_i \not \in S_j\}$ is an $X$-literal. Every $X$-literal in the NNF $\Delta^c[T_j]$ has the
form $\l' = \{x_i |  x_i \not \in S_k'\}$,
where $S_k' \subseteq S_j$
by the weak test-once property. Hence, $\l \models \l'$. Thus, by Lemma~\ref{lemma:iuq or distribute}, $\iuq$ distributes over the
or-nodes of $\Delta^c[T]$. Hence, we can compute $\iuq \instance \cdot \Delta^c[T]$ by simply applying $\iuq \instance$ to the literals of
$\Delta^c[T]$. If we do this using Proposition~\ref{prop:quantify-b},
we get the closed-from of $\iuq \instance \cdot \Delta^c[T]$ as shown 
in Proposition~\ref{prop:closed form}.
\end{proof}

\subsection*{Proposition~\ref{prop:gr properties}}

\begin{proof}[of Proposition~\ref{prop:gr properties}]
The given formula for the general reason has no negations 
so it is an NNF. 
Every \(X\)-literal in this NNF has the form 
$\l = \{x_i | x_i \not \in S_j\}$ where $\instance[X] \not \in S_j$. Hence, $\instance[X] \models \l$ which implies $\instance \models \l$. Thus, every literal in the NNF is consistent with 
instance $\instance$.

Every disjunction in this NNF has the form $\Gamma^c[T_j] \OR \l$ where $\l = \{x_i | x_i \not \in S_j\}$ is an $X$-literal. Every $X$-literal in the NNF $\Gamma^c[T_j]$ has the
form $\l' = \{x_i |  x_i \not \in S_k'\}$,
where $S_k' \subset S_j$
by the weak test-once property. Hence, $\l \models \l'$ and $\l \not = \l'$.
\end{proof}

\subsection*{Proposition~\ref{prop:PI_AND}}

\begin{proof}[of Proposition~\ref{prop:PI_AND}]
We prove the two directions.
\begin{description}
\item \(\tau \in \sPI(\alpha \AND \beta)\) only if 
\(\tau \in \subsum{\sPI(\alpha) \times \sPI(\beta)}\).
Suppose \(\tau \in \sPI(\alpha \AND \beta)\); that is,
\(\tau \models \alpha \AND \beta\) and
\(\tau \models \tau' \models \alpha \AND \beta\) for 
term \(\tau'\) only if \(\tau = \tau'\).
It suffices to show
(1)~\(\tau \in \sPI(\alpha) \times \sPI(\beta)\) and 
(2)~\(\tau \models \tau' \in  \sPI(\alpha) \times \sPI(\beta)\)
only if \(\tau = \tau'\). 

Let \(\tau_\alpha\) be the weakest term
such that \(\tau \models \tau_\alpha \models \alpha\) and
define \(\tau_\beta\) analogously. 
We next show that \(\tau_\alpha \in \sPI(\alpha)\),
\(\tau_\beta \in \sPI(\beta)\) and \(\tau = \tau_\alpha \AND \tau_\beta\)
which implies~(1).
Suppose \(\tau_\alpha \not \in \sPI(\alpha)\):
\(\tau_\alpha \models \tau_\alpha' \models \alpha\)
and \(\tau_\alpha \neq \tau_\alpha'\) for some term \(\tau_\alpha'\). 
This contradicts 
the definition of \(\tau_\alpha\)
so \(\tau_\alpha \in \sPI(\alpha)\).
We can similarly
show \(\tau_\beta \in \sPI(\beta)\). 
Finally, 
if \(\tau \neq \tau_\alpha \AND \tau_\beta\), 
then \(\tau \not \in \sPI(\alpha \AND \beta)\)
since 
\(\tau \models \tau_\alpha \AND \tau_\beta \models \alpha \AND \beta\),
a contradiction,
so \(\tau = \tau_\alpha \AND \tau_\beta\).
Hence,~(1) holds.
Suppose now~(2) does not hold: 
\(\tau \models \tau' \in  \sPI(\alpha) \times \sPI(\beta)\)
and \(\tau \neq \tau'\) for some term \(\tau'\). 
Let \(\tau_\alpha' \in \sPI(\alpha)\)
and \(\tau_\beta' \in \sPI(\beta)\) such that 
\(\tau' = \tau_\alpha' \AND \tau_\beta'\).  
Then \(\tau \models \tau' = \tau_\alpha' \AND \tau_\beta' \models \alpha \AND \beta\)
and \(\tau \neq \tau'\) which is a contradiction 
with \(\tau \in \sPI(\alpha \AND \beta)\).
Hence,~(2) holds and we have
\(\tau \in \subsum{\sPI(\alpha) \times \sPI(\beta)}\).

\item \(\tau \in \subsum{\sPI(\alpha) \times \sPI(\beta)}\)
only if \(\tau \in \sPI(\alpha \AND \beta)\).
Suppose \(\tau \in \subsum{\sPI(\alpha) \times \sPI(\beta)}\):
\(\tau \in \sPI(\alpha) \times \sPI(\beta)\) and 
\(\tau \models \tau' \in  \sPI(\alpha) \times \sPI(\beta)\)
only if \(\tau = \tau'\). We show
(1)~\(\tau \models \alpha \AND \beta\) and
(2)~\(\tau \models \tau' \models \alpha \AND \beta\) for
term \(\tau'\) only if \(\tau = \tau'\).

Let \(\tau = \tau_\alpha \AND \tau_\beta\) where
\(\tau_\alpha \in \sPI(\alpha)\) and \(\tau_\beta \in \sPI(\beta)\).
Then \(\tau \models \alpha \AND \beta\) which establishes~(1).
Suppose~(2) does not hold:
\(\tau \models \tau' \models \alpha \AND \beta\)
and \(\tau \neq \tau'\) for some term \(\tau'\).
Let \(\tau'\) be the weakest term satisfying the previous
property. Then \(\tau' \in \sPI(\alpha \AND \beta)\).
Let \(\tau_\alpha'\) be the weakest term
such that \(\tau'\models \tau_\alpha' \models \alpha\) and
define \(\tau_\beta'\) analogously. 
Then \(\tau_\alpha' \in \sPI(\alpha)\),
\(\tau_\beta' \in \sPI(\beta)\) 
and \(\tau' = \tau_\alpha' \AND \tau_\beta'\)
as shown in the first direction.
Hence, \(\tau \models \tau' \in \sPI(\alpha) \times \sPI(\beta)\).
Since \(\tau \neq \tau'\), we get a contradiction
with \(\tau \in \subsum{\sPI(\alpha) \times \sPI(\beta)}\).
Hence,~(2) holds and we have 
\(\tau \in \sPI(\alpha \AND \beta)\).
\end{description}
\end{proof}

\subsection*{Proposition~\ref{prop:PI_OR}}

\begin{proof}[of Proposition~\ref{prop:PI_OR}]
For literal \(\l\), \(\sPI(\l) = \{\l\}\).
Hence, what we need to show is 
$\sPI(\l \OR \beta) = \subsum{\{\l\} \cup \sPI(\beta)}$.
We next prove both directions. 
\begin{description}
\item $\tau \in \sPI(\l \OR \beta)$ only if
$\tau \in \subsum{\{\l\} \cup \sPI(\beta)}$.
Suppose $\tau \in \sPI(\l \OR \beta)$:
\(\tau \models \l \OR \beta\) and 
\(\tau \models \tau' \models \l \OR \beta\) for term \(\tau'\)
only if \(\tau = \tau'\).
We need to show
(1)~$\tau \in \{\l\} \cup \sPI(\beta)$ and
(2)~$\tau \models \tau' \in \{\l\} \cup \sPI(\beta)$ 
only if $\tau=\tau'$.

Our goal is to first prove either $\tau \models \l$ or $\tau \models \beta$, and then prove $\tau \in \{\l\} \cup \sPI(\beta)$. To prove  either $\tau \models \l$ or $\tau \models \beta$, showing $\tau \not \models \l$ only if $\tau \models \beta$ suffices.
Suppose $\tau \not \models \l$. Assume there exists a world $\w$ such that $\w \models \tau \models \l \OR \beta$ but $\w \not \models \beta$. If we find a contradiction, then all models of $\tau$ are models of $\beta$, i.e. $\tau \models \beta$, which is exactly what we want. Since $\tau \not \models \l$, either $\tau$ does not mention the variable of $\l$ or $\tau \models \l'$ for $\l \models \l'$ and $\l' \not = \l$. Both cases suggest there exists a world $\w'$ obtained from $\w$ by setting the state of the variable of $\l$ to some state not in $\l$ such that $\w' \models \tau$, $\w' \not \models \beta$ by the second property in Proposition~\ref{prop:gr properties}, and $\w' \not \models \l$. That is, $\w' \models \tau$  but $\w' \not \models \beta$ and $\w' \not \models \l$, so $\tau \not \models \l \OR \beta$. This is a contradiction. Thus, $\w \models \tau$ 
only if $\w \models \beta$, so $\tau \models \beta$. Therefore, $\tau \not \models \l$ only if $\tau \models \beta$. Equivalently,
$\tau \models \l$ or $\tau \models \beta$.
We can now prove (1) by considering two cases: $\tau \models \l$ and
$\tau \not \models \l$.
If $\tau \models \l$, then $\tau \models \l \models \l \OR \beta$ where $\l$ is a term, so $\tau = \l$, which means $\tau \in \{\l\} \cup \sPI(\beta)$. 
If $\tau \not \models \l$, then $\tau \models \beta$. For any term $\tau' \models \beta$, if $\tau \models \tau'$, then $\tau \models \tau' \models \beta \models \l \OR \beta$, which implies $\tau = \tau'$. Thus, $\tau \in \sPI(\beta)$. Hence, (1) must hold.

Suppose (2) does not hold: $\tau \models \tau' \in \{\l\} \cup \sPI(\beta)$ for a term $\tau' \not = \tau$. Then $\tau' \models \l \OR \beta$ since $\tau' \in \{\l\} \cup \sPI(\beta)$. We now have
$\tau \models \tau' \models \l \OR \beta$ and $\tau \not = \tau'$, which is a contradiction with \(\tau \in \sPI(\l \OR \beta)\). 
Hence, (2) must hold.

\item $\tau \in \subsum{\{\l\} \cup \sPI(\beta)}$ only if 
$\tau \in \sPI(\l \OR \beta)$.
Suppose $\tau \in \subsum{\{\l\} \cup \sPI(\beta)}$:
$\tau \in \{\l\} \cup \sPI(\beta)$, and 
$\tau \models \tau' \in \{\l\} \cup \sPI(\beta)$ 
only if $\tau=\tau'$. We next show
(1)~\(\tau \models \l \OR \beta\) and 
(2)~\(\tau \models \tau' \models \l \OR \beta\) for term \(\tau'\)
only if \(\tau = \tau'\).

If $\tau \in \{\l\}$, then $\tau \models \l$. If $\tau \in \sPI(\beta)$, then $\tau \models \beta$. Thus, $\tau \models \l \OR \beta$ follows from $\tau \in \{\l\} \cup \sPI(\beta)$, which establishes (1). Suppose (2) does not hold: $\tau \models \tau' \models \l \OR \beta$ and $\tau \not = \tau'$ for some term $\tau'$. Let $\tau'$ be the weakest term satisfying the previous property. Then $\tau' \in \sPI(\l \OR \beta)$. By the first direction, $\tau' \in {\l} \cup \sPI(\beta)$, so $\tau \models \tau' \in {\l} \cup \sPI(\beta)$ and $\tau \not = \tau'$,
which is a contradiction. Hence, (2)~holds 
and we have $\tau \in \sPI(\l \OR \beta)$. 
\end{description}
\end{proof}

\subsection*{Proposition~\ref{prop:PI inc-var-min}}

\begin{lemma}\label{lemma:GSR-implicant}
Let \(\gamma\) be a node in the NNF passed to 
Algorithm~\ref{alg:PI inc var-min}, $\GSR(.)$. 
The terms in $\GSR(\gamma)$ are implicants 
of $\gamma$.
\end{lemma}

\begin{proof}
The proof is by induction on the structure of the NNF passed to Algorithm~\ref{alg:PI inc var-min}.

\textit{Base case:} $\gamma$ is a literal or constant. This case is immediate.

\textit{Inductive step:} $\gamma = \alpha \AND \beta$. By the induction assumption, $\GSR(\alpha)$ are implicants of $\alpha$ and $\GSR(\beta)$ are implicants of $\beta$. For any $\tau_1 \in \GSR(\alpha)$ and $\tau_2 \in \GSR(\beta)$, we have $\tau_1 \AND \tau_2 \models \alpha \AND \beta$ so $S = \GSR(\alpha) \times \GSR(\beta)$ are implicants of $\alpha \AND \beta$. Algorithm~\ref{alg:PI inc var-min}
returns a subset of $S$ so the result holds.

\textit{Inductive step:} $\gamma = \alpha \OR \l$ where $\l$ is a literal. By the induction assumption, $\GSR(\alpha)$ are implicants of $\alpha$ and $\GSR(\l)$ is the implicant of $\l$. Then $\GSR(\alpha) \cup \GSR(\l)$ are implicants of $\alpha \OR \l$, since every implicant of $\alpha$ or $\l$ implies $\alpha \OR \l$. 
Algorithm~\ref{alg:PI inc var-min}
returns a subset of $S$ so the result holds.
\end{proof}

\begin{proof}[of Proposition~\ref{prop:PI inc-var-min}]

Let \(\Delta\) be the NNF passed in the first call
$\GSR(\Delta)$ to Algorithm~\ref{alg:PI inc var-min}.
We next prove two directions. 

\textit{First direction:}
If $\tau$ is a variable-minimal prime implicant of $\Delta$, then $\tau \in \GSR(\Delta)$. We prove this by contradiction.

We first note that
Algorithm~\ref{alg:PI inc var-min}, \(\GSR(\Delta)\), without Line~\ref{ln:GSR-vmin} (variable minimization) corresponds to
Algorithm~\ALGPI, \(\PI(\Delta)\), which computes the prime implicants of \(\Delta\).
Hence, we will say Algorithm~\ALGPI\ to mean Algorithm~\ref{alg:PI inc var-min} without Line~\ref{ln:GSR-vmin}.

Suppose now that $\tau$ is a variable-minimal prime implicant of $\Delta$ and
$\tau \not \in \GSR(\Delta)$. Since $\tau$ is a prime
implicant of $\Delta$, it must be equivalent to the conjunction of some terms \(S^*\) constructed by Algorithm~\ALGPI, where at least one of these terms is dropped
on Lines~\ref{ln:GSR-sub-and},~\ref{ln:GSR-sub-or} or~\ref{ln:GSR-vmin} of
Algorithm~\ref{alg:PI inc var-min}. By Lemma~\ref{lemma:GSR-implicant}, for each node $\gamma$ of the NNF $\Delta$, $\GSR(\gamma)$ are implicants of $\gamma$. Therefore, no prime implicant of $\gamma$ can be subsumed by any distinct term in $\GSR(\gamma)$. Thus, one of the terms $\tau^*$ in $S^*$ must have been removed by variable minimization on Line~\ref{ln:GSR-vmin} of
Algorithm~\ref{alg:PI inc var-min}; that is,  
not by the subsumption checks on Lines~\ref{ln:GSR-sub-and} or~\ref{ln:GSR-sub-or}
of the algorithm. 
Let $\Delta^*$ be the NNF node where the term $\tau^*$ is dropped by Algorithm~\ref{alg:PI inc var-min}. Then, there is a term ${\tau^+}$ generated by Algorithm~\ref{alg:PI inc var-min} at node $\Delta^*$ such that $\SetV({\tau^+}) \subset \SetV(\tau^*)$ and $\ivars{\Delta^*} \cap (\SetV(\tau^*) \setminus \SetV({\tau^+})) \not = \emptyset$. It follows that term $\tau$ is equivalent to the conjunction of term $\tau^*$ and some other terms $S^o$ constructed by
Algorithm~\ALGPI\ at nodes outside NNF
$\Delta^*$. Consider term $\tau'$ that is equivalent to the conjunction of ${\tau^+}$ and the terms
in $S^o$. Since the NNF $\Delta$ is locally fixated, the set of variables mentioned by $\tau'$ is equal to the union of the set of variables mentioned by the terms $S^o\cup\{\tau^+\}$. The same applies to term $\tau$ and terms $S^o\cup\{\tau^*\}$. Since $\ivars{\Delta^*} \cap (\SetV(\tau^*) \setminus \SetV({\tau^+})) \not = \emptyset$, we have $\SetV(\tau') \subset \SetV(\tau)$. Note that $\tau'$ is an implicant of $\Delta$. Thus, $\tau'$ is either 
a prime implicant of $\Delta$ or is subsumed by some distinct prime implicant of $\Delta$. Hence, $\tau$ cannot be a variable-minimal prime implicant of $\Delta$, which is a contradiction.

\textit{Second direction:} 
If $\tau \in \GSR(\Delta)$,
then $\tau$ is a variable-minimal prime implicant of $\Delta$. 
Suppose $\tau \in \GSR(\Delta)$. It suffices to show
(1)~$\tau \models \Delta$,
(2)~there is no prime implicant $\tau'$ of
$\Delta$ such that $\SetV(\tau') \subset \SetV(\tau)$, 
and (3)~there is no distinct prime implicant $\tau'$ of
$\Delta$ such that $\tau \models \tau'$.

Lemma~\ref{lemma:GSR-implicant} implies~(1) immediately.
We now show~(2).
By the first direction, $\GSR(\Delta)$ contains all variable-minimal prime implicants of $\Delta$. Thus, to prove (2), it suffices to prove that there does not exist a term $\tau' \in \GSR(\Delta)$ such that $\SetV(\tau') \subset \SetV(\tau)$. By the definition of $\ivars{.}$,
$\ivars{\Delta} = \SetV(\Delta)$ when $\Delta$
is the NNF passed to the first call to
Algorithm~\ref{alg:PI inc var-min}.
Thus, $\PIIVM{S}{\ivars{\Delta}}$ on
Line~\ref{ln:GSR-vmin} of the algorithm 
removes all terms from $S$ that are not
variable-minimal
in this case. Therefore, (2)~holds.
We next prove~(3) by contradiction. Assume there is a distinct prime implicant $\tau'$ of $\Delta$ such that $\tau \models \tau'$ (i.e., $\tau'$ subsumes $\tau$). Since, by the first direction, $\GSR(\Delta)$ contains all variable-minimal prime implicants of $\Delta$, $\tau'$ cannot be a variable-minimal prime implicant of $\Delta$; otherwise $\tau'$ will be in $\GSR(\Delta)$ so $\tau$ will not be in $\GSR(\Delta)$ as it will be removed by the subsumption checks on Line~\ref{ln:GSR-sub-and} or Line~\ref{ln:GSR-sub-or}, which is a contradiction. By~(2), no variable-minimal prime implicant of $\Delta$ has a strict subset of the variables in $\tau$. Therefore, $\SetV(\tau') \not \subseteq \SetV(\tau)$; otherwise $\tau'$ must be a variable-minimal prime implicant of $\Delta$. Note that $\tau'$ subsumes $\tau$ only if $\SetV(\tau') \subseteq \SetV(\tau)$. Therefore, $\tau'$ cannot subsume $\tau$, which is a contradiction. Hence, (3) holds.
\end{proof}

\subsection*{Proposition~\ref{prop:resolution}}

We first prove a dual of Proposition~\ref{prop:resolution} using
several lemmas. The dual is for the 
\textit{consensus} operation which can be used to compute the prime implicants of a DNF. The proof uses the same structure as the proof of
Theorem~3.5 in \cite{Crama2011BooleanF} which treats the Boolean
case of consensus.

\begin{definition}\label{def:consensus}
Let $\l_1 \AND \gamma_1$ and $\l_2 \AND \gamma_2$ be terms
where $\l_1$ and $\l_2$ are \(X\)-literals such that $\l_1 \not \models \l_2$ and $\l_2 \not \models \l_1$.
Then 
\(\gamma = (\l_1 \OR\l_2) \AND \gamma_1 \AND \gamma_2\) is
an \(X\)-consensus of the terms if \(\gamma \neq \bot\).
\end{definition}

We use $\consensus{\l_1 \AND \gamma_1, \l_2 \AND \gamma_2}{X}$
to denote the consensus of terms \(\l_1 \AND \gamma_1\) and 
\(\l_2 \AND \gamma_2\) on variable \(X\). We also use
$\EC(\Delta)$ to denote the result of closing DNF $\Delta$ 
under consensus and then removing all subsumed terms. Our
proofs will also use the following definition for
consensus over multiple terms (can be 
emulated by Definition~\ref{def:consensus} over two terms 
if we skip subsumed consensus).

\begin{definition}\label{def:multi consensus}
Let $\l_1 \AND \gamma_1, \ldots, \l_n \AND \gamma_n$ be terms
where \(\l_1,\ldots,\l_n\) are \(X\)-literals.
Then 
\(\gamma = (\BOR_{i = 1}^n \l_i) \AND \BAND_{i = 1}^n \gamma_i\) is
an \(X\)-consensus of the terms if \(\gamma \neq \bot\).
\end{definition}

\begin{lemma}\label{lemma:EC equivalence}\label{lemma:consensus_property}
We have $(\l_1 \OR\l_2) \AND \gamma_1 \AND \gamma_2 \models \lit_1 \AND \gamma_1 \OR \lit_2\AND \gamma_2$. 
Moreover, $\EC(\Delta)$ is equivalent to $\Delta$.
\end{lemma}

\begin{proof}
If \(\w \models (\l_1 \OR\l_2) \AND \gamma_1 \AND \gamma_2\),
then \(\w \models \l_1 \AND \gamma_1 \AND \gamma_2\) or
\(\w \models \l_2 \AND \gamma_1 \AND \gamma_2\). In either case,
\(\w \models \lit_1 \AND \gamma_1 \OR \lit_2\AND \gamma_2\).
Hence, $(\l_1 \OR\l_2) \AND \gamma_1 \AND \gamma_2 \models \lit_1 \AND \gamma_1 \OR \lit_2\AND \gamma_2$. This means that we can add to a
DNF \(\Delta\) the consensus of any of its terms without changing the models of \(\Delta\).
Hence, $\EC(\Delta)=\Delta$.
\end{proof}

\begin{lemma}\label{lemma:EC_lemma}
Let \(\tau\) be a simple term that mentions all variables in DNF $\Delta$.
If \(\tau \models \Delta\), then \(\tau \models \tau'\) for some 
term \(\tau'\) in $\Delta$ (that is, $\tau$ is subsumed by some term
in \(\Delta\)).
\end{lemma}

\begin{proof}
Since $\tau$ is simple and mentions all variables of $\Delta$, 
then \(\tau' | \tau = \top\) or \(\tau' | \tau = \bot\) for every
term \(\tau'\) in \(\Delta\). Since \(\tau \models \Delta\), 
\(\Delta | \tau = \top\) so
\(\tau' | \tau = \top\) for at least one term \(\tau'\) in \(\Delta\). 
This term must satisfy \(\tau \models \tau'\) and, hence, 
\(\tau\) is subsumed by \(\tau'\).
\end{proof}
Lemma~\ref{lemma:EC_lemma} does not hold if term \(\tau\) 
is not simple. Counterexample: \(\tau = x_{123}\) and 
\(\Delta = x_{12} + x_{23}\).

\begin{lemma}\label{lemma:EC_variable_lemma}
A prime implicant of DNF $\Delta$ can mention only variables 
mentioned by $\EC(\Delta)$.
\end{lemma}

\begin{proof}
If $\EC(\Delta)$ does not mention variable \(X\), then \(\Delta\)
does not depend on $X$ since \(\EC(\Delta)\) is 
equivalent to \(\Delta\) by Lemma~\ref{lemma:EC equivalence}.
Hence, any implicant of \(\Delta\) will remain an implicant of \(\Delta\)
if we drop any \(X\)-literal from it. Hence, a prime implicant of
\(X\) cannot mention variable \(X\).
\end{proof}

\begin{lemma}\label{lemma:consensus}
$\EC(\Delta)$ is the set of prime implicants for DNF $\Delta$.
\end{lemma}

\begin{proof}
We first show that every prime implicant of \(\Delta\) is in
$\EC(\Delta)$, and then show the second direction: every term
in \(\EC(\Delta)\) is a prime implicant of \(\Delta\).

To show the first direction, 
suppose $\tau_0$ is a prime implicant of $\Delta$ and
$\tau_0 \not \in \EC(\Delta)$. We next show a contradiction.
Let $S$ be the set of terms $\tau$ such that:
\begin{enumerate}
\item $\tau$ only mentions variables present in $\EC(\Delta)$.
\item $\tau \models \tau_0$.
\item $\tau$ is not subsumed by any term in $\EC(\Delta)$.
\end{enumerate}

By Lemma~\ref{lemma:EC_variable_lemma}, $\tau_0$ can only mention variables in $\EC(\Delta)$. Thus, $S$ must be non-empty because $\tau_0 \in S$. Let $\tau_m$ be the term in $S$ that mentions the largest number of variables (i.e. with the maximal length). 

\begin{description}
\item Case: $\tau_m$ mentions all variables of $\EC(\Delta)$.

Apply the following procedure which may change the value of $\tau_m$
but will keep the set $S$ intact:
\begin{description}
\item While $\tau_m \in S$:
\begin{itemize}
\item Write $\tau_m$ as 
\(x_{12\ldots n} \AND \tau_m'\) for some variable \(X\), term \(\tau_m'\)
and \(n > 1\). This can be done since $\tau_m$ is not a simple term by 
definition of \(S\) and Lemma~\ref{lemma:EC_lemma}.
\item For $i = 1, \ldots, n$:
If $x_i \AND \tau_m' \in S$, set 
$\tau_m$ to $x_i \AND \tau_m'$ and exit for-loop (variables of
\(\tau_m\) are invariant).
\item Exit while-loop if the for-loop did not set \(\tau_m\).
\end{itemize}
\end{description}
When the procedure terminates, \(\tau_m\) will be such that
$\tau_m \in S$ but for some variable $X$, $x_i \AND \tau_m' \not \in S$ for all $i$ in $1, \ldots, n$.
The procedure will always terminate because, by Lemma~\ref{lemma:EC_lemma}, simple terms that mention all variables cannot be in $S$. 
Since $\tau_m \in S$ upon termination, we have $\tau_m \models \tau_0$. And since $x_i \AND \tau_m' \models \tau_m$ for all $i$ in $1, \ldots, n$, we have $x_i \AND \tau_m' \models \tau_0$ for all $i$ in $1, \ldots, n$. Since $x_i \AND \tau_m' \not \in S$, 
$x_i \AND \tau_m'$ must be subsumed by some respective term in $\EC(\Delta)$ for each $i$. Since $\tau_m = x_{1 \ldots n} \AND \tau_m'$ is not subsumed by these respective terms, each $x_i \AND \tau_m'$ must be subsumed by some term $\alpha_i \in \EC(\Delta)$ that mentions state $x_i$.

Let $\beta_i$ be $\alpha_i$ but without its $X$-literal. Since $x_i \AND \tau_m' \models \alpha_i$ for all $i$, we have $\tau_m' \models \beta_i$ for all $i$. Hence, $\tau_m' \models \BAND_{i=1}^n \beta_i$. This means that the consensus of $\alpha_1, \ldots, \alpha_n$ on variable $X$ exists
since \(\alpha_i \AND \ldots \AND \alpha_n\) are consistent. 
Since $\tau_m' \models \BAND_{i=1}^n \beta_i$, we have
$x_{1\ldots n}\AND\tau_m' \models x_{1\ldots n} \AND \BAND_{i=1}^n  \beta_i$. 
And since each $\alpha_i$ mentions $x_i$, we have $x_{1\ldots n} \AND \BAND_{i=1}^n \beta_i \models \consensus{\alpha_1, \ldots, \alpha_i}{X}$. 
Therefore, $\tau_m \models x_{1\ldots n} \AND \BAND_{i=1} \beta_i \models \consensus{\alpha_1, \ldots, \alpha_i}{X}$
since $\tau_m = x_{1\ldots n} \AND \tau_m'$.
Hence, $\tau_m$ is subsumed by the consensus of $\alpha_1, \ldots, \alpha_i$ on variable $X$. Since $\alpha_i \in \EC(\Delta)$ for all $i$, their consensus must be subsumed by some term in $\EC(\Delta)$. Therefore, $\tau_m$ is subsumed by some term in $\EC(\Delta)$. This
contradicts $\tau_m \in S$.

\item Case: $\tau_m$ does not mention all variables of $\EC(\Delta)$.

Suppose $\tau_m$ does not mention variable $X$ which appears in
$\EC(\Delta)$.
Consider the terms $x_1 \AND \tau_m, \ldots, x_k \AND \tau_m$
where \(x_1, \ldots, x_k\) are the states of variable \(X\).
Since $\tau_m$ is a term in $S$ of maximal length, terms
$x_1 \AND \tau_m, \ldots, x_k \AND \tau_m$ cannot be in set $S$. 
Because $x_i \AND \tau_m$ satisfies the first two requirements of set $S$ for all $i$ between $1$ and $k$, $x_i \AND \tau_m$ must be subsumed by some term $\gamma_i \in \EC(\Delta)$. 
Since $\tau_m$ is not subsumed by $\gamma_i$ for any $i$, $\gamma_i$ must mention state $x_i$. Similarly, taking the consensus of $\gamma_1, \ldots, \gamma_k$ is allowed because $\tau_m \models \BAND_{i=1}^k (\gamma_i | x_i)$. Note that $\consensus{\gamma_1, \ldots, \gamma_k}{X}$ does not
mention variable $X$. Since $x_i \AND \tau_m \models \gamma_i$ 
for all $i$, we have
\(\tau_m \models \consensus{\gamma_1, \ldots, \gamma_k}{X}\).
Since $\gamma_i$ are all in $\EC(\Delta)$, $\consensus{\gamma_1, \ldots, \gamma_k}{X}$ must be subsumed by some term in $\EC(\Delta)$. This implies that $\tau_m$ is subsumed by some term in $\EC(\Delta)$, which contradicts the assumption that $\tau_m$ is in $S$. 
\end{description}
Our assumption that $\tau_0$ is a prime implicant of $\Delta$ but
$\tau_0 \not \in \EC(\Delta)$ leads to a contradiction in both cases above.
Thus, $\EC(\Delta)$ includes all prime implicants of $\Delta$.

We next show the second direction: 
every term in $\EC(\Delta)$ is a prime implicant of \(\Delta\).
Every term in $\EC(\Delta)$ is an implicant of $\Delta$ by Lemma~\ref{lemma:consensus_property}. Moreover, by definition of $\EC(\Delta)$,
no term in $\EC(\Delta)$ can subsume another term in $\EC(\Delta)$.
Hence, given the first direction, 
every term in $\EC(\Delta)$ is a prime implicants of $\Delta$.
\end{proof}

\begin{lemma}\label{lemma:duality}
The prime implicates of $\Delta$ are the negations of the prime implicates of $\NOT{\Delta}$.
\end{lemma}

\begin{proof}
This follows since \(\tau \models \NOT \Delta\) iff \(\Delta \models \NOT \tau\), and since \(\tau\) is equivalent to a term iff
\(\NOT \tau\) is equivalent to a clause.
\end{proof}

\begin{lemma}\label{lemma:resolution_consensus_duality}
For terms $\lit_1 \AND \tau_1$ and $\lit_2 \AND \tau_2$ where $\lit_1$, $\lit_2$ are \(X\)-literals, the negation of the consensus of $\lit_1 \AND \tau_1$ and $\lit_2 \AND \tau_2$ on $X$ is equivalent to the resolvent of $\NOT{\lit_1} \OR \NOT{\tau_1}$ 
and $\NOT{\lit_2} \OR \NOT{\tau_2}$ on $X$.
\end{lemma}

\begin{proof}
The consensus of $\lit_1 \AND \tau_1$ and $\lit_2 \AND \tau_2$ is $(\lit_1 \OR \lit_2) \AND \tau_1 \AND \tau_2$. 
The resolvent of $\NOT{\lit_1} \OR \NOT{\tau_1}$ 
and $\NOT{\lit_2} \OR \NOT{\tau_2}$ is 
$(\NOT{\lit_1} \AND \NOT{\lit_2}) \OR \NOT{\tau_1} \OR \NOT{\tau_2}$. 
Finally, 
$(\NOT{\lit_1} \AND \NOT{\lit_2}) \OR \NOT{\tau_1} \OR \NOT{\tau_2} = \NOT{(\lit_1 \OR \lit_2) \AND \tau_1 \AND \tau_2}$.
\end{proof}

\begin{proof}[of Proposition~\ref{prop:resolution}]
Let $\Delta$ be a CNF. By Lemma~\ref{lemma:consensus}, closing the DNF $\NOT{\Delta}$ under consensus and removing subsumed terms yields the prime implicants of $\NOT{\Delta}$. By Lemma~\ref{lemma:resolution_consensus_duality}, the negations of consensus generated while closing DNF $\NOT{\Delta}$ under consensus correspond 
to resolvents generated while closing CNF $\Delta$ under resolution. By Lemma~\ref{lemma:duality}, the prime implicates of $\Delta$ are the negations of the prime implicants of $\NOT{\Delta}$. Hence, closing $\Delta$ under resolution generates all the negations of the prime implicants of $\NOT{\Delta}$, which are the prime implicates of $\Delta$. Therefore, closing $\Delta$ under resolution and removing subsumed clauses yields exactly the prime implicates of $\Delta$.
\end{proof}

\subsection*{Proposition~\ref{prop:inc vd clauses}}

The proof of this proposition uses two lemmas
which effectively say that when applying resolution
to a locally fixated CNF, the variables of resolvents
grow monotonically. That is, if a clause \(\sigma^*\)
was derived using a clause \(\sigma\), then the variables
of \(\sigma^*\) are a superset of the variables of \(\sigma\).

\begin{lemma}\label{lemma:inc vd clauses}
Let $\alpha=\l_1 \OR \sigma_1$, 
$\beta=\l_2 \OR \sigma_2$ be two clauses which are locally fixated on some instance $\instance$. 
If $\l_1$ and $\l_2$ are \(X\)-literals, and if
\(\sigma\) is the \(X\)-resolvent
of clauses \(\alpha\) and \(\beta\), then
$\SetV(\sigma) = \SetV(\alpha) \cup \SetV(\beta)$.
\end{lemma}

\begin{proof}
Recall that a clause is a disjunction of literals
over distinct variables.
Suppose that $\l_1$ and $\l_2$ are \(X\)-literals.
If \(\sigma\) is the \(X\)-resolvent
of clauses \(\alpha\) and \(\beta\), then \(\sigma\)
is the clause equivalent to $(\l_1 \AND \l_2) \OR \sigma_1 \OR \sigma_2$ and \(\sigma \neq \top\).
Since $\alpha$ and $\beta$ are locally fixated on $\instance$, all literals in $\alpha$ and $\beta$ are consistent with $\instance$, so
$\l_1 \AND \l_2 \not = \bot$ and $X \in \SetV(\sigma)$. Since $\sigma \not = \top$, then
$\sigma_1 \OR \sigma_2 \not = \top$ so the
variables of the clause equivalent to \(\sigma_1 \OR \sigma_2\) are
$\SetV(\sigma_1) \cup \SetV(\sigma_2)$. Hence,  $\SetV(\sigma) = \SetV(\alpha) \cup \SetV(\beta)$.
\end{proof}

We will say that clause $\sigma^*$ is a \textit{descendant resolvent} of clause $\sigma$ if $\sigma^* = \sigma$ or if $\sigma^*$ was obtained by a sequence of resolutions that involved clause $\sigma$.

\begin{lemma}\label{lemma:descendant resolvent}
Let $S$ be a set of clauses which are locally fixated on some instance $\instance$, and 
let $S^*$ be 
the result of closing $S$ under resolution. If $\sigma^* \in S^*$ is a descendant resolvent of some $\sigma \in S$, then $\SetV(\sigma) \subseteq \SetV(\sigma^*)$.
\end{lemma}

\begin{proof}
This lemma follows directly from Lemma~\ref{lemma:inc vd clauses}.
\end{proof}

\begin{proof}[of Proposition~\ref{prop:inc vd clauses}]
We prove both directions.

\textit{First direction:} 
If $\sigma^*$ is a variable-minimal prime implicate of $S$, then $\sigma^*$ is a variable-minimal prime implicate of $S \setminus \{\sigma\}$.
Let $\sigma^*$ be a variable-minimal prime implicate of $S$. Our goal is to show that (1)~$\sigma^*$ is a prime implicate of $S \setminus \{\sigma\}$ and (2)~there does not exist another prime implicate ${{\sigma^{+}}}$ of $S \setminus \{\sigma\}$ such that $\SetV({{\sigma^{+}}}) \subset \SetV(\sigma^*)$.

To prove~(1), it suffices to show that 
(1a)~$\sigma^*$ is an implicate of $S \setminus \{\sigma\}$ and (1b)~$\sigma^*$ is not subsumed by any other implicate of $S \setminus \{\sigma\}$. 
To prove~(1a), we first recall that $ \SetV(\sigma) \supset \SetV(\sigma')$ for some 
clause $\sigma' \in S$ by the conditions of Proposition~\ref{prop:inc vd clauses}. Since $\sigma^*$ is a prime implicate of $S$, it must be derivable from $S$ using resolution by Proposition~\ref{prop:resolution}. 
Suppose there is a resolution
proof of $\sigma^*$ that involves clause \(\sigma\).
We will now show a contradiction, therefore establishing $\sigma^*$ as an implicate of $S\setminus\{\sigma\}$. 
First,
$\sigma^*$ is descendant resolvent of $\sigma$ in this case so $\SetV(\sigma^*) \supseteq \SetV(\sigma)$
by Lemma~\ref{lemma:descendant resolvent}. This implies
that $\SetV(\sigma^*) \supseteq \SetV(\sigma)
\supset \SetV(\sigma')$
for some clause $\sigma' \in S$. If $\sigma'$ is
a prime implicate of $S$, then 
$\sigma^*$ cannot be a variable-minimal 
prime implicate of $S$ since $\SetV(\sigma^*) 
\supset \SetV(\sigma')$. If $\sigma'$ is not
a prime implicate of $S$, then it must be subsumed by some prime implicate of $S$ which must mention a subset of the variables in $\sigma'$ so 
$\sigma^*$ cannot be a variable-minimal 
prime implicate of $S$. In either case, we have
a contradiction. Hence, $\sigma^*$
can be derived from $S$ using resolution without
involving clause $\sigma$. This means that
$\sigma^*$ is an implicate of $S\setminus\{\sigma\}$
so~(1a) holds.
We next show (1b) by contradiction. 
Suppose $\sigma^*$ is subsumed by some other implicate $\sigma^{**}$ of $S \setminus \{\sigma\}$. 
Then
$\sigma^*$ cannot be a prime implicate of $S$ as it is subsumed by $\sigma^{**}$
which must also be an implicate of $S$. This is a contradiction so $\sigma^*$ is not subsumed by any other implicate of $S \setminus \{\sigma\}$ and~(1b) holds. Hence,~(1) holds.

We now prove~(2) by contradiction. Suppose ${\sigma^{+}}$ is a
prime implicate of $S \setminus \{\sigma\}$ such that $\SetV({{\sigma^{+}}}) \subset \SetV(\sigma^*)$. 
Since ${\sigma^{+}}$ is an implicates of $S \setminus \{\sigma\}$, it is also an implicate of $S$.
Hence, either ${\sigma^{+}}$ is a prime implicate of $S$ or a clause that subsumes ${\sigma^{+}}$ (mentions a subset of ${\sigma^{+}}$'s variables) is a prime implicate of $S$. Either way, $\sigma^*$ cannot be a variable-minimal prime implicate of $S$, which is a contradiction, so~(2) holds.

\textit{Second direction:}
If $\sigma^*$ is a variable-minimal prime implicate of $S\setminus \{\sigma\}$, then $\sigma^*$ is a variable-minimal prime implicate of $S$.
Let $\sigma^*$ be a variable-minimal prime implicate of $S \setminus \{\sigma\}$. Our goal is to show that (1)~$\sigma^*$ is a prime implicate of $S$ and 
(2)~there does not exist another prime implicate $\sigma^{+}$ of $S$ such that $\SetV({{\sigma^{+}}}) \subset \SetV(\sigma^*)$.

To prove~(1), it suffices to show that 
(1a)~$\sigma^*$ is an implicate of $S$ and (1b)~$\sigma^*$ is not subsumed by any other prime implicate of $S$. Since $\sigma^*$ is an implicate of $S\setminus \{\sigma\}$, it must be an implicate of $S$ so~(1a) holds immediately. 
We next show~(1b). 
Since $\sigma^*$ is a prime implicate of $S \setminus \{\sigma\}$, it cannot be subsumed by any other prime implicate of $S \setminus \{\sigma\}$.
Suppose $\sigma^{**}$ is a prime implicate of $S$ but not 
a prime implicate of $S \setminus \{\sigma\}$.
Then $\sigma^{**}$ can be derived from $S$ using a resolution proof that involves $\sigma$. 
Hence, 
$\sigma^{**}$ is a descendent resolvent of $\sigma$ so
$\SetV(\sigma^{**}) \supseteq \SetV(\sigma)$ by Lemma~\ref{lemma:descendant resolvent}.
Moreover, 
$\SetV(\sigma) \supset \SetV(\sigma')$ 
for some clause $\sigma' \in S$ by
the conditions of Proposition~\ref{prop:inc vd clauses}.
Since
$\sigma^{**}$ subsumes $\sigma^*$ only if $\SetV(\sigma^{**}) \subseteq \SetV(\sigma^*)$ and
since 
$\SetV(\sigma') \subset \SetV(\sigma^{**})$,
then $\sigma^{**}$ cannot subsume $\sigma^*$;
otherwise, $\SetV(\sigma') \subset \SetV(\sigma^*)$, which implies there is a prime implicate of $S \setminus \{\sigma\}$ that subsumes $\sigma'$ and that mentions only a strict subset of the variables in $\sigma^*$, so $\sigma^*$ cannot be a variable-minimal prime implicate of $S \setminus \{\sigma\}$ which is a contradiction. As such,
$\sigma^*$ cannot be subsumed by any other prime implicate of $S$ so~(1b) and~(1) hold.

We next prove~(2) by contradiction. Suppose there is a prime implicate $\sigma^{+}$ of $S$ such that $\SetV(\sigma^{+}) \subset \SetV(\sigma^*)$.  
Then, $\sigma^+$ can be derived from $S$ using resolution
by Proposition~\ref{prop:resolution}. 
We consider two cases. 
First case: the resolution proof does not involve clause $\sigma$. Then $\sigma^+$
is an implicant of $S\setminus\{\sigma\}$.
Since some clause that subsumes $\sigma^+$ must be a prime implicate of $S\setminus\{\sigma\}$ and must
mention only a subset of the variables in $\sigma^+$, $\sigma^*$ 
cannot be a variable-minimal prime implicate of $S \setminus \{\sigma\}$ which is a contradiction.
Second case: the
resolution proof involves clause $\sigma$. 
In this case,
$\sigma^{+}$ is a descendant resolvent of $\sigma$ so $\SetV(\sigma^{+}) \supseteq \SetV(\sigma)$ by Lemma~\ref{lemma:descendant resolvent}. This further implies
$\SetV(\sigma^{+}) \supset \SetV(\sigma')$ for some clause $\sigma' \in S$ by the conditions of Proposition~\ref{prop:inc vd clauses}, and
also 
$\SetV(\sigma^{*}) \supset \SetV(\sigma^{+}) \supset \SetV(\sigma')$. 
Since $\sigma' \in S \setminus \{\sigma\}$, some prime implicate of $S \setminus \{\sigma\}$ must subsume $\sigma'$ and must mention only a subset of its variables. Therefore, $\sigma^*$ cannot be a variable-minimal prime implicate of $S \setminus \{\sigma\}$, a contradiction. 
We get a contradiction in both cases, so~(2) holds.
\end{proof}

\clearpage
\section{Path Explanations}
\label{sec:path exp diff}

Consider the decision tree in Figure~\ref{fig:PathFigure1} which classifies the instance 
\( (x_1\!\!=\!\!1,x_2\!\!=\!\!1,x_3\!\!=\!\!1,x_4\!\!=\!\!1)\) as Y using the red path.
This path corresponds to the term $(x_1 \in \{1, 2\}, x_2 \in \{1, 2\}, x_3 \in \{1\}, x_4 \in \{1\})$ which implies the
formula \(\Delta_Y\) for class Y.
This term is normally viewed as an explanation for the decisions on instances that follow this path.
However, the shorter term $(x_1 \in \{1, 2\}, x_2 \in \{1, 2\}, x_3 \in \{1\})$ also implies the class formula \(\Delta_Y\) and can therefore be viewed as a better explanation since
feature $x_4$ is irrelevant to such decisions.
This phenomena was observed in~\cite{DBLP:journals/jair/IzzaIM22} which introduced the notion of an abductive path explanation (APXp): a minimal subset of the literals on a path that implies the corresponding class formula. The APXp is a syntactic notion
as it depends on the specific decision tree. That is,
two different decision trees that represent the same classifier may lead to different APXps. This is in contrast
to the notion of a GSR that we propose which is a semantic notion that depends only on the underlying classifier (i.e., its class formulas). 
That is, two distinct decision trees that represent the same classifier always lead to the same GSRs for any instance.

 \begin{figure}[h]
        \centering
        \scalebox{0.8}{
        \begin{tikzpicture}[
        roundnode/.style={circle ,draw=black, thick},
        squarenode/.style={rectangle, draw=black, thick},
        ]
        %Nodes
        \node[roundnode] (x4) {\Large $x_4$};
        \node[squarenode]     (Y0)        [below=of x4, xshift = -1.2cm, yshift = 0.8cm] {\Large Y};
        \node[roundnode]     (x3)        [below=of x4, xshift = 1.2cm, yshift = 0.8cm] {\Large $x_3$};
        \node[roundnode]     (x1l)       [below=of x3, xshift = -1.8cm, yshift = 0.8cm] {\Large $x_1$};
        \node[squarenode]    (Y1)       [below=of x1l, xshift = -0.8cm, yshift = 0.8cm] {\Large Y};
        \node[squarenode]    (N1)       [below=of x1l, xshift = 0.8cm, yshift = 0.8cm] {\Large N};
        
        \node[roundnode]     (x2)       [below=of x3, xshift = 1.8cm, yshift = 0.8cm] {\Large $x_2$};
        \node[roundnode]     (x1r)       [below=of x2, xshift = -0.8cm, yshift = 0.8cm] {\Large $x_1$};
        \node[squarenode]    (N2)       [below=of x2, xshift = 0.8cm, yshift = 0.8cm] {\Large N};
        \node[squarenode]    (Y2)       [below=of x1r, xshift = -0.8cm, yshift = 0.8cm] {\Large Y};
        \node[squarenode]    (N3)       [below=of x1r, xshift = 0.8cm, yshift = 0.8cm] {\Large N};

        %Lines
        \draw[-latex, thick] (x4.240) -- node [anchor = center, xshift = -4mm, yshift = 2mm] {$\{0\}$} (Y0.north);
        \draw[-latex, thick, red] (x4.300) -- node [anchor = center, xshift = 4mm, yshift = 2mm] {$\{1\}$} (x3.north);
        
        \draw[-latex, thick] (x3.240) -- node [anchor = center, xshift = -2mm, yshift = 2mm] {$\{0\}$} (x1l.north);
        \draw[-latex, thick, red] (x3.300) -- node [anchor = center, xshift = 2mm, yshift = 2mm] {$\{1\}$} (x2.north);

        \draw[-latex, thick] (x1l.240) -- node [anchor = center, xshift = -4mm, yshift = 1.5mm] {$\{1\}$} (Y1.north);
        \draw[-latex, thick] (x1l.300) -- node [anchor = center, xshift = 6mm, yshift = 1.5mm] {$\{2, 3\}$} (N1.north);

        \draw[-latex, thick, red] (x2.240) -- node [anchor = center, xshift = -6mm, yshift = 1.5mm] {$\{1, 2\}$} (x1r.north);
        \draw[-latex, thick] (x2.300) -- node [anchor = center, xshift = 4mm, yshift = 1.5mm] {$\{3\}$} (N2.north);

        \draw[-latex, thick, red] (x1r.240) -- node [anchor = center, xshift = -6mm, yshift = 1.5mm] {$\{1, 2\}$} (Y2.north);
        \draw[-latex, thick] (x1r.300) -- node [anchor = center, xshift = 4mm, yshift = 1.5mm] {$\{3\}$} (N3.north);
        \end{tikzpicture}
        }
        \caption{A decision tree with two classes: Y and N. Variables $x_1$, $x_2$ are ternary. Variables $x_3$, $x_4$ are binary. \label{fig:PathFigure1}}
\end{figure}
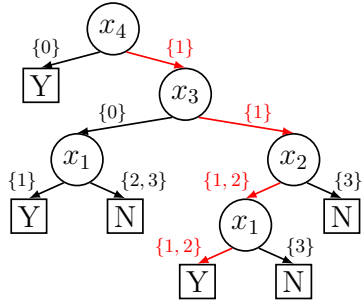

For the decision tree in Figure~\ref{fig:PathFigure1},
the decision on instance \((x_1\!\!=\!\!1,x_2\!\!=\!\!1,x_3\!\!=\!\!1,x_4\!\!=\!\!1)\)
has a GSR,
($x_1 \in \{1\}, x_2 \in \{1, 2\}$), which
does not correspond
to any APXp of any path in the decision tree.
Moreover, this GSR generates
the SR $(x_1\!\!=\!\!1,x_2\!\!=\!\!1$) as ensured
by our Proposition~\ref{prop:gsr_no_lost}.

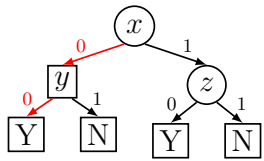
\begin{figure}[h]
        \centering
        \scalebox{0.8}{
        \begin{tikzpicture}[
        roundnode/.style={circle ,draw=black, thick},
        squarenode/.style={rectangle, draw=black, thick},
        ]
        %Nodes
        \node[roundnode] (x) {\Large $x$};
        \node[squarenode]     (y)        [below=of x, xshift = -1.2cm, yshift = 0.7cm] {\Large $y$};
        \node[squarenode]    (Y1)       [below=of y, xshift = -0.6cm, yshift = 0.7cm] {\Large Y};
        \node[squarenode]    (N1)       [below=of y, xshift = 0.6cm, yshift = 0.7cm] {\Large N};
        
        \node[roundnode]     (z)        [below=of x, xshift = 1.2cm, yshift = 0.7cm] {\Large $z$};
        \node[squarenode]    (Y2)       [below=of z, xshift = -0.6cm, yshift = 0.7cm] {\Large Y};
        \node[squarenode]    (N2)       [below=of z, xshift = 0.6cm, yshift = 0.7cm] {\Large N};
        %Lines
        \draw[-latex, thick, red] (x.240) -- node [anchor = center, xshift = -2mm, yshift = 1.5mm] {$0$} (y.north);
        \draw[-latex, thick] (x.300) -- node [anchor = center, xshift = 2mm, yshift = 1.5mm] {$1$} (z.north);
        
        \draw[-latex, thick, red] (y.240) -- node [anchor = center, xshift = -2mm, yshift = 1.5mm] {$0$} (Y1.north);
        \draw[-latex, thick] (y.300) -- node [anchor = center, xshift = 2mm, yshift = 1.5mm] {$1$} (N1.north);
        \draw[-latex, thick] (z.240) -- node [anchor = center, xshift = -2mm, yshift = 1.5mm] {$0$} (Y2.north);
        \draw[-latex, thick] (z.300) -- node [anchor = center, xshift = 2mm, yshift = 1.5mm] {$1$} (N2.north);
        \end{tikzpicture}
        }
        \caption{A decision tree with two classes: Y and N. All variables are binary.}\label{fig:PathFigure2}
\end{figure}

For another example, consider Figure~\ref{fig:PathFigure2} in which all variables are binary so one does not need to
go beyond simple explanations that are subsets of instances.
The instance $(\eql{x}{0}, \eql{y}{0}, \eql{z}{0})$ is classified as Y using the red path. This decision has two SRs, \((\eql{x}{0}, \eql{y}{0})\) and \((\eql{y}{0}, \eql{z}{0})\). The APXp for the red path is $(\eql{x}{0}, \eql{y}{0})$. The APXps for the other paths are $(\eql{x}{0}, \eql{y}{1})$, $(\eql{x}{1}, \eql{z}{0})$, $(\eql{x}{1}, \eql{z}{1})$. None correspond to the SR $(\eql{y}{0}, \eql{z}{0})$. 
This shows that APXps cannot even generate
all simple explanations (i.e., subsets of the instance).
In contrast, Proposition~\ref{prop:gsr_no_lost} guarantees that every SR will be generated by some GSR. 
A similar argument applies to the notion of contrastive path explanation (CPXp) proposed
in~\cite{DBLP:journals/jair/IzzaIM22}.
See also Example~6 in~\cite{DBLP:journals/jair/IzzaIM22}
for a related discussion of limitations. 

\section{Numeric Features}
\label{sec:numeric}

\begin{figure}[tb]
\centering
 \resizebox{0.2\textwidth}{!}{%
        \begin{tikzpicture}[
        roundnode/.style={text width = 0.55cm, circle ,draw=black, thick},
        squarednode/.style={rectangle, draw=black, thick},
        ]
        %Nodes
        \node[squarednode]     (AGE1)                              {\Large Age};
        \node[squarednode]     (BMI2)       [below=of AGE1, xshift = -0.8cm, yshift = 0.4cm] {\Large BMI};
        \node[squarednode]     (AGE2)       [below=of AGE1, xshift = 0.8cm, yshift = -1.2cm] {\Large Age};
        \node[roundnode]       (NO3)   [, below=of BMI2, xshift = -0.5cm, yshift = 0.4cm] {\Large No};
        \node[roundnode]       (YES3)   [below=of BMI2, xshift = 0.5cm, yshift = 0.4cm] {\Large Yes};
        \node[squarednode]     (BMI31)       [below=of AGE2, xshift = -1.45cm, yshift = 0.4cm] {\Large BMI};
        \node[squarednode]     (BMI32)       [below=of AGE2, xshift = 0.15cm, yshift = -1.2cm] {\Large BMI};
        \node[roundnode]       (NO41)   [below=of BMI31, xshift = -0.5cm, yshift = 0.4cm] {\Large No};
        \node[roundnode]       (YES42)   [below=of BMI31, xshift = 0.5cm, yshift = 0.4cm] {\Large Yes};
        \node[roundnode]       (NO43)   [below=of BMI32, xshift = -0.2cm, yshift = 0.4cm] {\Large No};
        \node[roundnode]       (YES44)   [below=of BMI32, xshift = -1.2cm, yshift = 0.4cm] {\Large Yes};
        
        %Lines
        \draw[-latex, thick] (AGE1.240) -- node [anchor = center, xshift = -6mm, yshift = 1mm] {$<18$} (BMI2.north);
        \draw[-latex, thick] (AGE1.300) -- node [anchor = center, xshift = 6mm, yshift = 1mm] {$\ge18$} (AGE2.north);
        \draw[-latex, thick] (BMI2.240) -- node [anchor = center, xshift = -6mm] {$<30$} (NO3.north);
        \draw[-latex, thick] (BMI2.300) -- node [anchor = center, xshift = 6mm] {$\ge30$} (YES3.north);
        \draw[-latex, thick] (AGE2.240) -- node [anchor = center, xshift = -5mm, yshift = 1mm] {$<40$} (BMI31.north);
        \draw[-latex, thick] (AGE2.300) -- node [anchor = center, xshift = 5mm, yshift = 1mm] {$\ge40$} (BMI32.north);
        \draw[-latex, thick] (BMI31.240) -- node [anchor = center, xshift = -5mm] {$<27$} (NO41.north);
        \draw[-latex, thick] (BMI31.300) -- node [anchor = center, xshift = 5mm] {$\ge27$} (YES42.north);
        \draw[-latex, thick] (BMI32.240) -- node [anchor = center, xshift = -6mm] {$\ge25$} (YES44.north);
        \draw[-latex, thick] (BMI32.300) -- node [anchor = center, xshift = 5mm] {$<25$} (NO43.north);
        \end{tikzpicture}
}%
\quad
 \resizebox{0.35\textwidth}{!}{%
        \begin{tikzpicture}[
        roundnode/.style={text width = 0.55cm, circle ,draw=black, thick},
        squarednode/.style={rectangle, draw=black, thick},
        ]
        %Nodes
        \node[squarednode]     (AGE1)                              {\Large Age};
        \node[squarednode]     (BMI2)       [below=of AGE1, xshift = -1.05cm, yshift = 0.4cm] {\Large BMI};
        \node[squarednode]     (AGE2)       [below=of AGE1, xshift = 0.8cm, yshift = -1.2cm] {\Large Age};
        \node[roundnode]       (NO3)   [, below=of BMI2, xshift = -0.5cm, yshift = 0.4cm] {\Large No};
        \node[roundnode]       (YES3)   [below=of BMI2, xshift = 0.5cm, yshift = 0.4cm] {\Large Yes};
        \node[squarednode]     (BMI31)       [below=of AGE2, xshift = -2.85cm, yshift = 0.4cm] {\Large BMI};
        \node[squarednode]     (BMI32)       [below=of AGE2, xshift = 0.2cm, yshift = -1.2cm] {\Large BMI};
        \node[roundnode]       (NO41)   [below=of BMI31, xshift = -0.5cm, yshift = 0.4cm] {\Large No};
        \node[roundnode]       (YES42)   [below=of BMI31, xshift = 0.5cm, yshift = 0.4cm] {\Large Yes};
        \node[roundnode]       (YES44)   [below=of BMI32, xshift = -1.3cm, yshift = 0.4cm] {\Large Yes};
        \node[roundnode]       (NO43)   [below=of BMI32, xshift = 0cm, yshift = 0.4cm] {\Large No};
        
        %Lines
        \draw[-latex, thick] (AGE1.240) -- node [anchor = center, xshift = -8mm, yshift = 1mm] {$[0, 18)$} (BMI2.north);
        \draw[-latex, thick] (AGE1.300) -- node [anchor = center, xshift = 10mm, yshift = 8mm] {$[18, 40), [40, \infty)$} (AGE2.north);
        \draw[-latex, thick] (BMI2.240) -- node [anchor = center, xshift = -21mm] {$[0, 25), [25, 27), [27, 30)$} (NO3.north);
        \draw[-latex, thick] (BMI2.300) -- node [anchor = center, xshift = 7mm] {$[30, \infty)$} (YES3.north);
        \draw[-latex, thick] (AGE2.240) -- node [anchor = center, xshift = -8mm, yshift = 2mm] {$[18, 40)$} (BMI31.north);
        \draw[-latex, thick] (AGE2.300) -- node [anchor = center, xshift = 6mm, yshift = 1mm] {$[40, \infty)$} (BMI32.north);
        \draw[-latex, thick] (BMI31.240) -- node [anchor = center, xshift = -15mm] {$[0, 25), [25, 27)$} (NO41.north);
        \draw[-latex, thick] (BMI31.300) -- node [anchor = center, xshift = 15mm] {$[27, 30), [30, \infty)$} (YES42.north);
        \draw[-latex, thick] (BMI32.240) -- node [anchor = center, xshift = -23mm] {$[25, 27), [27, 30), [30, \infty)$} (YES44.north);
        \draw[-latex, thick] (BMI32.300) -- node [anchor = center, xshift = 8mm] {$[0, 25)$} (NO43.north);
        \end{tikzpicture}
}%
\caption{Numeric features (left) and their discretization (right). \label{fig:continuous-dg}}
\end{figure}

GSRs and GNRs are
particularly significant when explaining the decisions of classifiers with numeric features, such as decision trees
and random forests.
Consider the decision tree in
Figure~\ref{fig:continuous-dg}(left). 
One can discretize its 
numeric features to yield the decision tree in Figure~\ref{fig:continuous-dg}(right) as is commonly practiced. 
For example,
$\Age$ is discretized into three intervals: $[0, 18), [18, 40)$ and $[40, \infty)$ so it can be treated as a ternary discrete variable. Similarly, $\BMI$ is discretized into four intervals: $[0, 25), [25, 27), [27, 30), [30, \infty)$. The numeric and discrete decision trees are equivalent as they make the same 
decision on every instance. This follows since two distinct instances will be classified equally by the numeric decision 
tree if the point values of their features fall into the same intervals. 

The decision on instance $(\eql{\Age}{42} \AND \eql{\BMI}{28})$ is $\yes$. 
To explain this decision, one usually works with the
discrete decision tree which views this as the discrete 
instance $(\eql{\Age}{[40,\infty)} \AND \eql{\BMI}{[27,30)})$, which can be notated equivalently 
as $(\Age \GE 40) \AND (27 \LE \BMI \LT 30)$.
There is only one SR for the decision
on this instance, which is
$(\Age \GE 40) \AND (27 \LE \BMI \LT 30)$; that is,
the instance itself.
But there are two GSRs:
$(\Age \GE 18 \AND \BMI \GE 27)$ and 
$(\Age \GE 40 \AND \BMI \GE 25)$ which are
significantly more informative. 
SRs are quite limited 
in this case as they can only reference simple literals that
appear in the instance: 
%$(\eql{\Age}{42})$ and $(\eql{\BMI}{28})$.
$\eql{\Age}{[40,\infty)}$ and 
$\eql{\BMI}{[27,30)}$.
GSRs can reference any literal implied by
the instance, such as $\Age \IN \{[18, 40),[40,\infty)\}$,
which allows them to provide more informative explanations.

The NRs for the above decision are $\Age\GE 40$ and $27 \LE \BMI \LT 30$. All we can learn from the second one, 
as an example, is that it is possible to flip the decision by 
changing $\BMI$ to some value $\NIN [27, 30)$. 
If we change $\BMI$ to $32$, keeping $\Age$ the same, this NR is violated but the decision is not changed (we are only guaranteed that \textit{some} change that violates the NR will flip the decision).
In contrast, 
the GNRs are $\Age \GE 18$ and $\BMI \GE 25$ which come with stronger guarantees as mentioned earlier.
For example, 
the second GNR, $\BMI \GE 25$,
tells us that changing $\BMI$ to \(\LT 25\), while keeping $\Age$ the same, is guaranteed to flip the decision which is significantly more informative.

\section{More on General Reasons, GSRs and GNRs}
\label{sec:var-min}

Suppose $\Gamma$ is a 
general reason; 
$\tau_1, \ldots, \tau_n$ are the GSRs 
(variable-minimal prime implicants of $\Gamma$),
and $\sigma_1, \ldots, \sigma_m$ are the GNRs (variable-minimal prime implicates of $\Gamma$).
Then it is possible that $\Gamma \neq \BOR_{i=1}^n \tau_i$,
$\Gamma \neq \BAND_{i=1}^m \sigma_i$
and/or
$\BOR_{i=1}^n \tau_i \neq \BAND_{i=1}^m \sigma_i$.
To illustrate this, 
consider the class formula $\Delta = x_1 \AND y_1 \OR x_{12} \AND y_{12} \AND z_1$ and instance $\instance = x_1 \AND y_1 \AND z_1$. The general reason is $\iuq \instance \cdot \Delta = \Delta$. The only GSR is $x_1 \AND y_1$ and the GNRs are $x_{12}$ and $y_{12}$. We have,
$\Delta \neq x_1 \AND y_1$;
$\Delta \neq x_{12} \AND y_{12}$;
and
$x_1 \AND y_1 \neq x_{12} \AND y_{12}$.
This is different from the case for simple explanations where the disjunction of SRs, the conjunction of NRs, and the complete reason are all equivalent. Therefore, neither GSRs nor GNRs capture all the information contained in the general reason,
which suggest that general reasons may have futher
applications beyond GSRs and GNRs.

We now turn to another key observation.
Suppose $\Gamma$ is a general reason for instance $\instance$ and let $\sigma$
be one of its prime implicates ($\sigma$ is not necessarily variable-minimal and, hence, may not be a GNR). We can minimally change instance $\instance$
to violate $\sigma$ yet without necessarily flipping the
decision on $\instance$. This can never happen though if $\sigma$ is variable-minimal (by
Definition~\ref{def:gnr} and Proposition~\ref{prop:PI-IP}).

Consider the following example with 
ternary variables $X, Y, Z$,
instance $\instance = x_1 \AND y_1 \AND z_1$ 
and its class formula $\Delta = \NOT \Delta_n$
where 
\[
\begin{aligned}[t] \Delta_n = 
&\ (x_1 \AND y_2 \AND z_3)\ \OR
(x_1 \AND y_3 \AND z_2)\ \OR 
(x_1 \AND y_3 \AND z_3)\ \OR\\
&\ (x_2 \AND y_1 \AND z_2)\ \OR 
(x_3 \AND y_1 \AND z_2)\ \OR
(x_3 \AND y_1 \AND z_3)\ \OR\\
&\ (x_2 \AND y_2 \AND z_1)\ \OR
(x_2 \AND y_3 \AND z_1)\ \OR
(x_3 \AND y_2 \AND z_1).
\end{aligned}
\]
The general reason $\iuq \instance \cdot \Delta$ 
for the decision on instance $\instance$ is
\[
\begin{aligned}[t]
&\ (\bot \OR y_{13} \OR z_{12}) \AND
(\bot \OR y_{12} \OR z_{13}) \AND
(\bot \OR y_{12} \OR z_{12}) \AND\\
&\ (x_{13} \OR \bot \OR z_{13}) \AND 
 (x_{12} \OR \bot \OR z_{13}) \AND
(x_{12} \OR \bot \OR z_{12}) \AND \\
&\ (x_{13} \OR y_{13} \OR \bot) \AND
(x_{13} \OR y_{12} \OR \bot) \AND
(x_{12} \OR y_{13} \OR \bot).
\end{aligned}
\]
which simplifies to
\[
\iuq \instance \cdot \Delta 
= (y_{12} \OR z_1) \AND (y_1 \OR z_{12}) \AND (x_{12} \OR z_1) \AND (x_1 \OR z_{13}) \AND (x_{13} \OR y_1) \AND (x_1 \OR y_{13}).
\]
Note that $\sigma = x_1 \OR y_1 \OR z_1$ is a prime implicate of $\iuq I \cdot \Delta$ which can be obtained by resolving $y_1 \OR z_{12}$ with $x_1 \OR z_{13}$ on variable $Z$. 
However, any instance $\instance'$ that does not satisfy $\sigma$ is a model of $\Delta$. This follows since 
$\instance' \models {\NOT \sigma} = x_{23} \AND y_{23} \AND z_{23}$
and all models of $\NOT{\Delta} = \Delta_n$ contain $x_1$, $y_1$ or $z_1$. Therefore, violating the prime implicate $\sigma$ of the general reason does
not flip the decision. Note further
that this prime implicate $\sigma$ is not variable-minimal (i.e., not a GNR) since $\sigma' = y_{12} \OR z_1$ is also a prime implicate of the general reason and
$\vars(\sigma') \subset \vars(\sigma)$.

\end{document}